\documentclass{article}
\usepackage[textwidth=400pt,centering]{geometry}
\usepackage{url}

\usepackage[utf8]{inputenc} 
\usepackage[T1]{fontenc}    
\usepackage{booktabs}       
\usepackage{nicefrac}       
\usepackage{microtype}      
\usepackage{xcolor}
\usepackage{xcite}
\definecolor{mydarkred}{rgb}{0.6,0,0}
\definecolor{mydarkgreen}{rgb}{0,0.6,0}

\usepackage{amsmath,amsthm,amsfonts,amssymb}
\usepackage{bm}
\usepackage{bbm}
\usepackage{comment}
\usepackage{tablefootnote}
\usepackage{multirow}
\usepackage{hhline}
\usepackage{graphicx}

\usepackage{algorithm,algorithmic}

\usepackage[colorlinks,linkcolor=mydarkred,citecolor=mydarkgreen,pdfborder={0 0 1},citebordercolor={0 1 0}]{hyperref}

\usepackage[font=small,labelfont=bf]{caption}
\usepackage{subcaption}
\usepackage{url}
\usepackage[round,comma,numbers]{natbib}

\DeclareMathOperator*{\argmin}{\mathrm{arg\,min}}
\DeclareMathOperator*{\argmax}{\mathrm{arg\,max}}

\newtheorem{theorem}{Theorem}
\newtheorem{lemma}[theorem]{Lemma}

\usepackage{xr}
\begin{document}
    \title{Online Multiclass Classification Based on Prediction Margin for Partial Feedback}

\author{Takuo Kaneko \\The University of Tokyo\\RIKEN\\ \texttt{kaneko@ms.k.u-tokyo.ac.jp}\\
\vspace{-1mm}\\
Issei Sato \\The University of Tokyo\\RIKEN\\ \texttt{sato@k.u-tokyo.ac.jp} \\
\vspace{-1mm}\\
Masashi Sugiyama\\RIKEN\\The University of Tokyo\\\texttt{sugi@k.u-tokyo.ac.jp}
}

\maketitle

    \begin{abstract}
    We consider the problem of online multiclass classification with partial feedback, where an algorithm predicts a class for a new instance in each round and only receives its correctness. Although several methods have been developed for this problem, recent challenging real-world applications require further performance improvement. In this paper, we propose a novel online learning algorithm inspired by recent work on \emph{learning from complementary labels}, where a complementary label indicates a class to which an instance does \emph{not} belong. This allows us to handle partial feedback deterministically in a margin-based way, where the prediction margin has been recognized as a key to superior empirical performance. We provide a theoretical guarantee based on a cumulative loss bound and experimentally demonstrate that our method outperforms existing methods which are non-margin-based and stochastic.
\end{abstract}
    \section{Introduction}
Starting with the \emph{perceptron} \cite{Perceptron}, research on online classification has been extensively conducted \cite{Fobos, rda, Adagrad}. Methods that use the prediction margin, which indicates the difference between the score of a classifier and a classification boundary, such as \emph{passive-aggressive} (PA) \cite{PA}, \emph{confidence-weighted} (CW) \cite{CW} and their variants \cite{Arow, SCW}, have been shown to achieve better empirical performance. In addition, some of these prediction-margin based methods have theoretical guarantee based on mistake bounds in adversarial cases.

Some methods have been extended to online multiclass classification \cite{PA, Arow, SPA}. In this multiclass setting, the prediction margin is defined as the difference in scores between classes, and these methods update the classifier on the basis of it. The algorithms assume that an instance and its correct label are received in each round (which is called \emph{full feedback}), and the classifier is updated with them. However, there are many cases where it is easy to know whether the prediction was correct or not (which is called \emph{partial feedback}), but hard to obtain correct labels in all rounds.

There has been some research on online multiclass classification with partial feedback, e.g, Banditron \cite{Banditron}, Confidit \citep{Confidit}, exp\_grad \citep{Exp_grad}, Newtron \citep{Newtron}, the second order banditron algorithm (SOBA) \citep{SOBA}, bandit passive-aggressive (BPA) \citep{BPA} and confidence-weighted bandit learning (CWB) \citep{CWB}. However, their empirical performance is not well in practice. Banditron, Confidit, exp\_grad, Newtron and SOBA employ classical perceptron-based algorithms for update rules. BPA and CWB employ PA-based and CW-based algorithms, respectively, which are known as prediction-margin based algorithms. However, they handle multiclass problems in a one-versus-rest way and apply update rules to each classifier independently. Thus they are not based on the prediction margin in terms of multiclass classification. In addition, the previous research commonly uses some exploration strategies in the label space for training the classifier. Banditron, exp\_grad, Newtron, BPA and CWB conduct their explorations in the manner of an $\epsilon$-greedy method. Confidit conducts exploration on the basis of an upper confidence bound. 

In this paper, we propose a deterministic prediction-margin based algorithm for online multiclass classification with partial feedback. When the prediction is correct, we will use the update rule of support-class passive aggressive \cite{SPA}, which is a state of the art PA based method for online multiclass classification with full feedback. On the other hand, for the case where the proposed label is incorrect, we propose a new update rule, inspired by \emph{learning from complementary labels}\footnote{A complementary label indicates a class to which an instance does not belong.} \cite{complementary}. Our contributions in this paper can be summarized as follows:
\begin{itemize}
    \item We propose a deterministic prediction-margin based algorithm for online multiclass classification with partial feedback, by combining \emph{support-class passive aggressive} \cite{SPA} for correct prediction and \emph{learning from complementary labels} \cite{complementary} for incorrect prediction (Section \ref{proposed_section}).
    \item We theoretically show the convergence of the proposed method by deriving a cumulative square loss bound (Section \ref{analysis}).
    \item We experimentally demonstrate the superior performance of the proposed method compared with existing methods for partial feedback (Section \ref{experiments}).
\end{itemize}
    \section{Preliminary}
In this section, we formulate the problem of online multiclass classification with partial feedback.
\subsection{Problem setting}
In ordinary online multiclass classification setting, in each round $t$, the algorithm receives an instance ${\bf x}_t \in \mathbb{R}^d$ and predicts its label denoted by ${\hat y}_t \in \{1, \ldots, K\}$, where $d$ is the dimension of the feature vectors and $K$ is the number of classes. Then, the algorithm receives the correct label $y_t$ and improves the classifier if necessary. 

In contrast, in the partial feedback setting, the algorithm chooses a proposed label ${\tilde y}_t$ after making a prediction ${\hat y}_t$ and then asks an oracle whether ${\tilde y}_t$ is correct or not. The goal is to reduce the number of mistaken proposed labels 
\begin{equation}
    \sum_{t=1}^T \mathbbm{1}[{\tilde y}_t \not= y_t]
\end{equation}
as much as possible. 

\subsection{Model}
Our algorithm uses a linear-in-parameter model that is used by the existing online learning algorithms \citep{PA, CW, Arow, SCW, Perceptron}. We consider $K$ weight vectors ${\bf w}_i \in \mathbb{R}^d$, for $i = 1, \ldots, K$ and give a score ${\bf w}_i^{\top}{\bf x}$ for class $i$ of instance ${\bf x}$, where $\top$ denotes the transpose. We define a classifier $f: \mathbb{R}^d \rightarrow \{1, \ldots, K\}$ that predicts the label for ${\bf x} \in \mathbb{R}^d$ as follows:
\begin{equation}
f({\bf x}) = \argmax_{i=1,\ldots,K} {\bf w}_i^{\top}{\bf x}.
\end{equation}
We denote the parameters ${\bf w}_i$ at round $t$, as ${\bf w}_{i, t}$.

    \section{Proposed Method}
\label{proposed_section}

In this section, we introduce our proposed algorithm for online multiclass classification with partial feedback. Our algorithm is based on online passive-aggressive algorithms \citep{PA}, which are based on prediction-margin and perform well in online classification problems. 

In the $t$-th round, the algorithm receives an instance ${\bf x}_t \in \mathbb{R}^d$ such that $\|{\bf x}_t\| = R$ , where $R$ is a constant value. Then, it predicts its label ${\hat y}_t$ as follows:
\begin{equation}
{\hat y}_t = \argmax_{i \in \{1, \ldots, K\}} {\bf w}_{i, t}^{\top}{\bf x}_t.
\end{equation}
Regarding the proposed label ${\tilde y}_t$, our algorithm always behaves deterministically; that is, it always selects ${\hat y}_t$ as ${\tilde y}_t$, whereas the previous algorithms \citep{Banditron, Confidit, Exp_grad, Newtron, BPA, CWB, SOBA} may conduct exploration in several ways. Next, the algorithm receives $M_t = \mathbf{1}\{y_t = {\tilde y}_t \}$ where $\mathbf{1}\{\text{condition}\}=1$ if the condition is satisfied, and $0$ otherwise. The algorithm behaves differently in accordance with $M_t$.

\subsection{Update rule for the wrong proposed label}
When $M_t = \mathrm{False}$, i.e, ${\tilde y}_t$ is not the correct label, this label can be regarded as a \emph{complementary label} \citep{complementary}.
We propose an online algorithm for complementary labels. Here, we define the loss $\ell_t$ when ${\tilde y}_t$ is the \emph{complementary label} in round $t$ as follows:
\begin{equation}
\label{CPA_loss}
\ell_t = \min_{i \in \{1, \ldots, K\} \backslash \{{\tilde y}_t\}} 1 - {\bf w}_{i, t}^{\top}{\bf x}_t + {\bf w}_{{\tilde y}_t, t}^{\top}{\bf x}_t. 
\end{equation}
This loss corresponds to the minimum margin between the scores of the incorrect class ${\tilde y}_t$ and the other classes.
Here, since ${\tilde y}_t = \arg \max_{i \in \{1, \ldots, K\}} {\bf w}_{i, t}^{\top}{\bf x}_t$, the following is satisfied:
\begin{equation}
\label{CPA_l}
\ell_t \geq 1.
\end{equation}

For this case where the proposed label is wrong, we formulate the following optimization problem for round $t$ with a hyperparameter $\beta \in (0, 1]$:
\begin{equation}
\begin{split}
\label{cop}
& {\bf w}_{1, t+1}, \ldots, {\bf w}_{K, t+1}  \\
= & \argmin_{{\bf w}_1, \ldots, {\bf w}_K \in \mathbb{R}^d}  \sum_{i \in \{1, \ldots, K\}} \|{\bf w}_i - {\bf w}_{i, t}\|^2  \\
\mathrm{s.t.} & \min_{i \in \{1, \ldots, K\}} \left(1 - {\bf w}_{i}^{\top}{\bf x}_t + {\bf w}_{{\tilde y}_t}^{\top}{\bf x}_t \right) \leq (1-\beta)\ell_t,  \\
& \hspace{0.3cm} {\bf w}_{i}^{\top}{\bf x}_t - {\bf w}_{j}^{\top}{\bf x}_t = {\bf w}_{i, t}^{\top}{\bf x}_t - {\bf w}_{j, t}^{\top}{\bf x}_t, \,\, \forall i, j \in \{1, \ldots, K\} \backslash \{{\tilde y}_t\}.
\end{split}
\end{equation}
The algorithm knows the label ${\tilde y}_t$ is not correct. Therefore the weight vectors ${\bf w}_1, \ldots, {\bf w}_K$ are updated so that the score of ${\tilde y}_t$ is not the highest, that is, the prediction margin between the scores of ${\tilde y}_t$ and another class $i$, ${\bf w}_{{\tilde y}_t}^{\top}{\bf x}_t - {\bf w}_{i}^{\top}{\bf x}_t$ becomes smaller. 

In contrast, the algorithm does not know the true label. Therefore, the weight vectors are updated in such a way that the prediction margins between the scores of labels other than ${\tilde y}_t$, do not change. This corresponds to the second constraint in (\ref{cop}). As a result, the algorithm focuses on the margin between the score of ${\tilde y}_t$ and the second highest score,
\begin{equation}
\begin{split}
& {\bf w}_{{\tilde y}_t}^{\top}{\bf x}_t - \max_{i \in \{1, \ldots, K\} \backslash \{{\tilde y}_t\}} {\bf w}_{i}^{\top}{\bf x}_t = \min_{i \in \{1, \ldots, K\} \backslash \{{\tilde y}_t\}} \left(- {\bf w}_{i}^{\top}{\bf x}_t + {\bf w}_{{\tilde y}_t}^{\top}{\bf x}_t \right).
\end{split}
\end{equation}

In the following, we derive a closed update rule for the optimization problem (\ref{cop}).

First, the following lemma holds for the form of the update.
\begin{lemma}
\label{update_form}
The update rule for the optimization problem (\ref{cop}) is expressed for some $\tau_1, \ldots, \tau_K \in \mathbb{R}$ as:
\begin{equation}
\begin{split}
\label{lemma1}
{\bf w}_{i,t+1} = {\bf w}_{i,t} + \tau_i{\bf x}_t \,\, \forall i \in \{1, \ldots, K\}.
\end{split}
\end{equation}
\end{lemma}
\begin{proof}
Let the optimal solution of (\ref{cop}) be ${\bf w}_i^{\star}$. Then ${\bf w}_i^{\star} - {\bf w}_{i, t}$ can be expressed for some ${\bf z}_i$ such that ${\bf x}_t^{\top} {\bf z}_i = 0$ as follows
\begin{equation}
\begin{split}
{\bf w}_i^{\star} - {\bf w}_{i, t} = \tau_i{\bf x}_t + {\bf z}_i.
\end{split}
\end{equation}
Substituting this equality for the optimization problem (\ref{cop}), we obtain
\begin{equation}
\begin{split}
\label{cop_1}
& \argmin_{{\bf z}_1, \ldots, {\bf z}_K \in \mathbb{R}^d, \tau_1, \ldots, \tau_K \in \mathbb{R}} \sum_{i \in \{1, \ldots, K\}} \tau_i^2\|{\bf x}_t\|^2 + \|{\bf z}_i\|^2  \\
\mathrm{s.t.} & \min_{i \in \{1, \ldots, K\}} (1 - {\bf w}_{i, t}^{\top}{\bf x}_t + {\bf w}_{{\tilde y}_t, t}^{\top}{\bf x}_t) - \tau_i\|{\bf x}_t\|^2 + \tau_{{\tilde y}_t}\|{\bf x}_t\|^2 \leq (1-\beta)\ell_t,  \\
& \tau_i = \tau_j,  \,\, \forall i, j \in \{1, \ldots, K\} \backslash \{{\tilde y}_t\}.
\end{split}
\end{equation}
The objective function is minimized by $\|z_i\|^2=0$, for all $i$ and the lemma is proven.
\end{proof}

By Lemma \ref{update_form}, (\ref{cop}) can be rewritten as the following optimization problem:
\begin{equation}
\begin{split}
\label{op1}
& \argmin_{\tau_1, \ldots, \tau_K \in \mathbb{R}} \frac{1}{2} \sum_{i \in \{1, \ldots, K\}} \tau_i\|{\bf x}_t\|^2  \\
\mathrm{s.t.} & \min_{i \in \{1, \ldots, K\} \backslash \{{\tilde y}_t\}} (1 - {\bf w}_{i, t}^{\top}{\bf x}_t + {\bf w}_{{\tilde y}_t, t}^{\top}{\bf x}_t) - (\tau_i - \tau_{\tilde{y}_t})\|{\bf x}_t\|^2 \leq (1-\beta)\ell_t, \\
& \tau_i = \tau_j \,\, \forall i, j \in \{1, \ldots, K\} \backslash \{{\tilde y}_t\}.
\end{split}
\end{equation}
From the second constraint of (\ref{op1}), for $i \in \{1, \ldots, K\} \backslash \{{\tilde y}_t\}$, we can denote $\tau_i = \tau$ for some $\tau \in \mathbb{R}$. Consequently, the optimization problem to be solved is
\begin{equation}
\begin{split}
\label{op2}
\argmin_{\tau, \tau_{{\tilde y}_t} \in \mathbb{R}} & \frac{1}{2}(K-1)\tau^2 + \frac{1}{2}\tau_{{\tilde y}_t}^2 \\
\mathrm{s.t.} \,\, & \beta\ell_t - (\tau - \tau_{{\tilde y}_t})\|{\bf x}_t\|^2 \leq 0.
\end{split}
\end{equation}
Note that $\ell_t$ is defined as (\ref{CPA_loss}).

If $\ell_t = 0$, then $\tau = \tau_{{\tilde y}_t} = 0$ satisfies the constraint in (\ref{op2}) and is the optimal solution. Therefore, we concentrate on the case $\ell_t > 0$. Here, we introduce a Lagrange multiplier $\lambda \geq 0$ and define the Lagrangian function of (\ref{op2}) as follows:
\begin{equation}
\label{lagrange}
L(\tau, \tau_{{\tilde y}_t}, \lambda) = \frac{1}{2}(K-1)\tau^2 + \frac{1}{2}\tau_{{\tilde y}_t}^2 + \lambda(\beta\ell_t - (\tau - \tau_{{\tilde y}_t})\|{\bf x}_t\|^2).
\end{equation}
Since its derivative with respect to $\tau$ is zero for an optimal solution of (\ref{op2}), we have
\begin{equation}
\begin{split}
&\frac{\partial L(\tau, \tau_{{\tilde y}_t}, \lambda)}{\partial \tau} = (K-1)\tau - \lambda\|{\bf x}_t\|^2 = 0 \\
& \frac{\partial L(\tau, \tau_{{\tilde y}_t}, \lambda)}{\partial \tau_{{\tilde y}_t}} = \tau_{{\tilde y}_t} + \lambda\|{\bf x}_t\|^2 = 0,
\end{split}
\end{equation}
and obtain
\begin{equation}
\begin{split}
\label{tau}
& \tau = \frac{\lambda\|{\bf x}_t\|^2}{K-1}, \\
& \tau_{{\tilde y}_t} = -\lambda\|{\bf x}_t\|^2.
\end{split}
\end{equation}
Substituting (\ref{tau}) for (\ref{lagrange}) yields
\begin{equation}
\label{lagrange2}
L(\lambda) = \frac{K\|{\bf x}_t\|^4}{2(K-1)}\lambda^2 - \beta\ell_t\lambda.
\end{equation}
Then, taking the derivative of (\ref{lagrange2}) with respect to $\lambda$ and setting it to zero, we obtain
\begin{equation}
\label{lambda}
\lambda = \frac{1}{\|{\bf x}_t\|^4}\frac{K-1}{K}\ell_t.
\end{equation}
Substituting (\ref{lambda}) for (\ref{tau}) yields
\begin{equation}
\begin{split}
& \tau = \frac{1}{K}\frac{\beta\ell_t}{\|{\bf x}_t\|^2}, \\
& \tau_{{\tilde y}_t} = -\frac{K-1}{K}\frac{\beta\ell_t}{\|{\bf x}_t\|^2}.
\end{split}
\end{equation}

Finally, we obtain the following update rule:
\begin{equation}
\begin{split}
\label{update}
{\bf w}_{i, t+1} = 
\begin{cases}
{\bf w}_{i, t} + \frac{1}{K}\frac{\beta\ell_t}{\|{\bf x}_t\|^2}{\bf x}_t \,\, & (i \in \{1, \ldots, K\} \backslash \{{\tilde y}_t\}), \\
{\bf w}_{{\tilde y}_t, t} - \frac{K-1}{K}\frac{\beta\ell_t}{\|{\bf x}_t\|^2}{\bf x}_t & (i = {\tilde y}_t).
\end{cases}
\end{split}
\end{equation}

We discuss the choice of $\beta$ in Section \ref{analysis}. Intuitively, $\beta$ plays a role in adjusting the step-size. The closer $\beta$ is to $1$, the more aggressive the update is. On the other hand, the closer it is to $0$, the more passive the update is.

\subsection{Update rule for the correct proposed label}
When $M_t = \mathrm{True}$, i.e., the proposed label ${\tilde y}_t$ is the correct label, this round is regarded as an ordinary situation and we can use an existing online learning algorithm, the \emph{support-class passive aggressive} (SPA) algorithm \citep{SPA}. We briefly review the SPA algorithm below. 

First, the loss for the class $i \in \{1, \ldots, K\} \backslash \{{\tilde y}_t\}$ at round $t$ is defined as
\begin{equation}
\label{SPA_loss2}
\ell_{i, t} = \max(1 + {\bf w}_{i, t}^{\top}{\bf x}_t - {\bf w}_{{\tilde y}_t, t}^{\top}{\bf x}_t, 0) \,\, (i \in \{1, \ldots, K\} \backslash \{{\tilde y}_t\}),
\end{equation}
and the loss at round $t$ is defined as
\begin{equation}
\label{SPA_loss}
\ell_{t} = \max_{i \in \{1, \ldots, K\} \backslash \{{\tilde y}_t\}} \ell_{i,t} \,\, (i \in \{1, \ldots, K\} \backslash \{{\tilde y}_t\}).
\end{equation}
The loss $\ell_t$ corresponds to the margin between the scores of the correct class ${\tilde y}_t$ and all other classes.
Here, because ${\tilde y}_t = \argmax_{i \in \{1, \ldots, K\}} {\bf w}_{i, t}^{\top}{\bf x}_t$, the following is satisfied:
\begin{equation}
\label{SPA_l}
0 \leq \ell_t \leq 1.
\end{equation}

Let $\sigma(k)$ be the $k$-th class when $\ell_{i, t}$ is sorted in descending order. Then, the support class $S_t$, which is the set of classes whose parameters are updated, is determined as follows:
\begin{equation}
\label{support_class}
S_t = \left\{\sigma(k) \middle| \,\sum_{j=1}^{k-1}\ell_{\sigma(j), t} < k\ell_{\sigma(k), t} \right\}.
\end{equation}

The update rule of SPA is expressed on the basis of $S_t$ and $\ell_{i, t}$ defined above:
\begin{equation}
\begin{split}
{\bf w}_{{\tilde y}_t, t+1} = 
\begin{cases}
{\bf w}_{i, t} - \frac{1}{|S_t|+1}\left(\sum_{j \in S_t} \ell_{j, t}\right) {\bf x}_t \,\, & (i = {\tilde y}_t), \\
{\bf w}_{i, t} + \left(\ell_{i, t} - \sum_{j \in S_t} \frac{\ell_{j, t}}{|S_t|+1} \right) {\bf x}_t \,\, & (i \in S_t), \\
{\bf w}_{i, t} \,\, & (i \notin S_t).
\end{cases}
\end{split}
\end{equation}

Please refer to \citet{SPA} for the derivation of the update rules, etc.

The whole algorithm is shown in Algorithm \ref{algorithm}. Since our algorithm extends the SPA algorithm to the complementary label case, we call it complementary SPA (CSPA).

\begin{algorithm}[t]
\caption{CSPA algorithm with partial feedback}
\label{algorithm}
\begin{algorithmic}
\REQUIRE $\beta \in (0, 1]$.
\ENSURE ${\bf w}_{i, 1} \leftarrow {\bf 0} \in \mathbb{R}^d \,\, \forall i \in \{1, \ldots, K\}$.
\FOR{$t=1,2, \ldots, T$}
\STATE Receive an instance ${\bf x}_t \in \mathbb{R}^d$ : $\|{\bf x}_t\| = R$.
\STATE Predict label \, ${\hat y}_t = \argmax_{i \in \{1,\ldots,K\}} {\bf w}_{i, t}^{\top}{\bf x}_t$.
\STATE Set the \emph{proposed} label ${\tilde y}_t \leftarrow {\hat y}_t$.
\STATE Get the feedback $M_t = \{y_t = {\tilde y}_t\}$.
\IF {$M_t = \mathrm{False}$}
\STATE Calculate the loss \\
\hspace{0.5cm} $\ell_t = \min_{i \in \{1, \ldots, K\}} 1 - {\bf w}_{i, t}^{\top}{\bf x}_t + {\bf w}_{{\tilde y}_t, t}^{\top}{\bf x}_t$. 
\STATE Update \\
\hspace{0.5cm} ${\bf w}_{i, t+1} = {\bf w}_{i, t} + \frac{1}{K}\frac{\beta\ell_t}{\|{\bf x}_t\|^2}{\bf x}_t$ \\
\hspace{3cm} $(i \in \{1, \ldots, K\} \backslash \{{\tilde y}_t\})$.
\STATE Update \, ${\bf w}_{\tilde{y}, t+1} = {\bf w}_{{\tilde y}, t} - \frac{K-1}{K}\frac{\beta\ell_t}{\|{\bf x}_t\|^2}{\bf x}_t$.
\ELSE
\STATE Calculate the loss $\ell_{i, t}$ and $\ell_{t}$ according to (\ref{SPA_loss2}) and (\ref{SPA_loss}).
\STATE $S_t = \emptyset$.
\WHILE {$\sum_{j = 1}^{|S_t|} \frac{\ell_{\sigma(j), t}}{|S_t|+1} < \ell_{\sigma(|S_t|)}$}
\STATE $S_t = S_t \cup \{\sigma(|S_t|)\}$.
\ENDWHILE
\STATE Update \, ${\bf w}_{{\tilde y}, t+1} = {\bf w}_{{\tilde y}, t} - \frac{1}{|S_t|+1}\left(\sum_{j \in S_t} \ell_{j, t}\right) {\bf x}_t$.
\STATE Update \\ 
\hspace{0.5cm} ${\bf w}_{i, t+1} = {\bf w}_{i, t} + \left(\ell_{i, t} - \sum_{j \in S_t} \frac{\ell_{j, t}}{|S_t|+1} \right) {\bf x}_t$ \\ 
\hspace{5.7cm} $(i \in S_t)$.
\STATE Update \, ${\bf w}_{i, t+1} = {\bf w}_{i, t} \,\, (i \notin S_t)$.
\ENDIF
\ENDFOR
\end{algorithmic}
\end{algorithm}

    \section{Theoretical Analysis}
\label{analysis}

In this section, we derive a cumulative square loss bound for CSPA. Theoretical analyses of online prediction-margin based algorithms have been considered on the basis of bounds on the number of mistakes \citep{Arow} or cumulative square loss \citep{PA, SPA}. We follow the analysis presented in \citet{PA} and \citet{SPA} for deriving the cumulative square loss of the CSPA algorithm in the partial feedback setting.

In the CSPA algorithm, the proposed label depends on the classification function in each round. Therefore, we consider an adversarial case where there is no assumption about the distribution of the sequence of data, similar to what is done in \citet{Banditron, Newtron}, so that CSPA can cope with this situation. 

Recall the definition of the loss function $\ell_t$ is defined in Section \ref{proposed_section} by (\ref{CPA_loss}) and (\ref{SPA_loss}) as follows:
\begin{equation}
\begin{split}
\ell_t = \begin{cases}
\max_{i \in \{1, \ldots, K\} \backslash \{{\tilde y}_t\}} \max(1 + {\bf w}_{i, t}^{\top}{\bf x}_t - {\bf w}_{{\tilde y}_t, t}^{\top}{\bf x}_t, 0) \,\,& 
\\ \hspace{5.6cm} ({\tilde y}_t = y_t),\\
\min_{i \in \{1, \ldots, K\} \backslash \{{\tilde y}_t\}} \, 1 - {\bf w}_{i, t}^{\top}{\bf x}_t + {\bf w}_{{\tilde y}_t, t}^{\top}{\bf x}_t \,\, & 
\\ \hspace{5.6cm} ({\tilde y}_t \not= y_t). 
\end{cases}
\end{split}
\end{equation}

We have the following bound on the cumulative square loss.
\begin{theorem}
\label{main_theorem}
Let $({\bf x}_1, y_1), \ldots, ({\bf x}_T, y_T)$ be a sequence  where $y_t \in \{1, \ldots, K\}$ is the correct label of ${\bf x}_t \in \mathbb{R}^d$ such that $\|{\bf x}_t\| = R$ for all $t$. Let ${\bf u}_1, \ldots, {\bf u}_K$ be vectors satisfying the following conditions for all $t$:
\begin{equation}
\label{assumption}
\max_{y' \in \{1, \ldots, K\} \backslash \{y_t\}} \sum_{y'' \not= y_t, y'} ({\bf u}_{y''}^{\top}{\bf x}_t - {\bf u}_{y'}^{\top}{\bf x}_t) \leq \alpha \,\, (0 \leq \exists \alpha < 1),
\end{equation}
and define the loss $\ell_t^{\star}$ as follows:
\begin{equation}
\ell_t^{\star} = \max_{i \in \{1, \ldots, K\} \backslash \{y_t\}} \max(1 + {\bf u}_{i, t}^{\top}{\bf x}_t - {\bf u}_{y_t, t}^{\top}{\bf x}_t, 0).
\end{equation}
When $\beta$ is set to be
\begin{equation}
\label{beta_assumption}
\beta = \frac{1-\alpha}{K-1},
\end{equation}
the cumulative square loss $\ell_t$ of CSPA on this sequence is bounded from above as follows:
\begin{equation}
\begin{split}
\sum_{t=1}^{T} \ell_t^2 \leq & \Biggl(\frac{K(K-1)}{(1-\alpha)^2}\sqrt{\sum_{i=1}^{T} (\ell_t^{\star})^2} + \frac{R\sqrt{K(K-1)}}{1-\alpha}\sqrt{\sum_{i \in \{1, \ldots, K\}}\|{\bf u}_i\|^2} \Biggr) ^2.
\end{split}
\end{equation}
\end{theorem}
The proof is in Appendix A in the supplementary material. 

This theorem indicates that for any ${\bf u}_1, \ldots, {\bf u}_K$ satisfying (\ref{assumption}), the square loss of CSPA can be bounded by the loss of ${\bf u}_1, \ldots, {\bf u}_K$. This bound is the same order as the \emph{online passive-aggressive algorithm} and the \emph{support class passive-aggressive algorithm} \citep{PA, SPA}, with respect to $T$.

In fact, the derived bound of the cumulative square loss of $\ell_t$ upper bounds the mistake bounds as follows:
\begin{equation}
\begin{split}
    \sum_{t=1}^T \ell_t^2 & \geq \sum_{t=1}^T \ell_t^2 \mathbbm{1}[\ell_t \geq 1] \\
    & \geq \sum_{t=1}^T \mathbbm{1}[\ell_t \geq 1] \\
    & = \sum_{t=1}^T \mathbbm{1}[\tilde{y}_t \not= y_t].
\end{split}
\end{equation}
Therefore, it means that the derived bound can also bound the number of mistakes.

The constraint (\ref{assumption}) requires that the differences between the scores of classes other than the correct label $y_t$, are relatively small. This may seem to be a strong constraint but ${\bf u}_1, \ldots, {\bf u}_K$ satisfying (\ref{assumption}) always exists because the LHS of (\ref{assumption}) goes to zero for ${\bf u}_1, \ldots, {\bf u}_K$ sufficiently close to zero vectors.

In terms of the regret for the adversarial cases, our derived bound does not assure the superiority to the existing methods in \citet{Banditron, Confidit, Exp_grad, Newtron}. Nevertheless, this gives a theoretical guarantee to a prediction-margin based algorithm for the partial feedback setting for the first time.


    \section{Experiments}
\label{experiment}
\begin{figure}[t]
\centering
\includegraphics[keepaspectratio, scale=0.25]{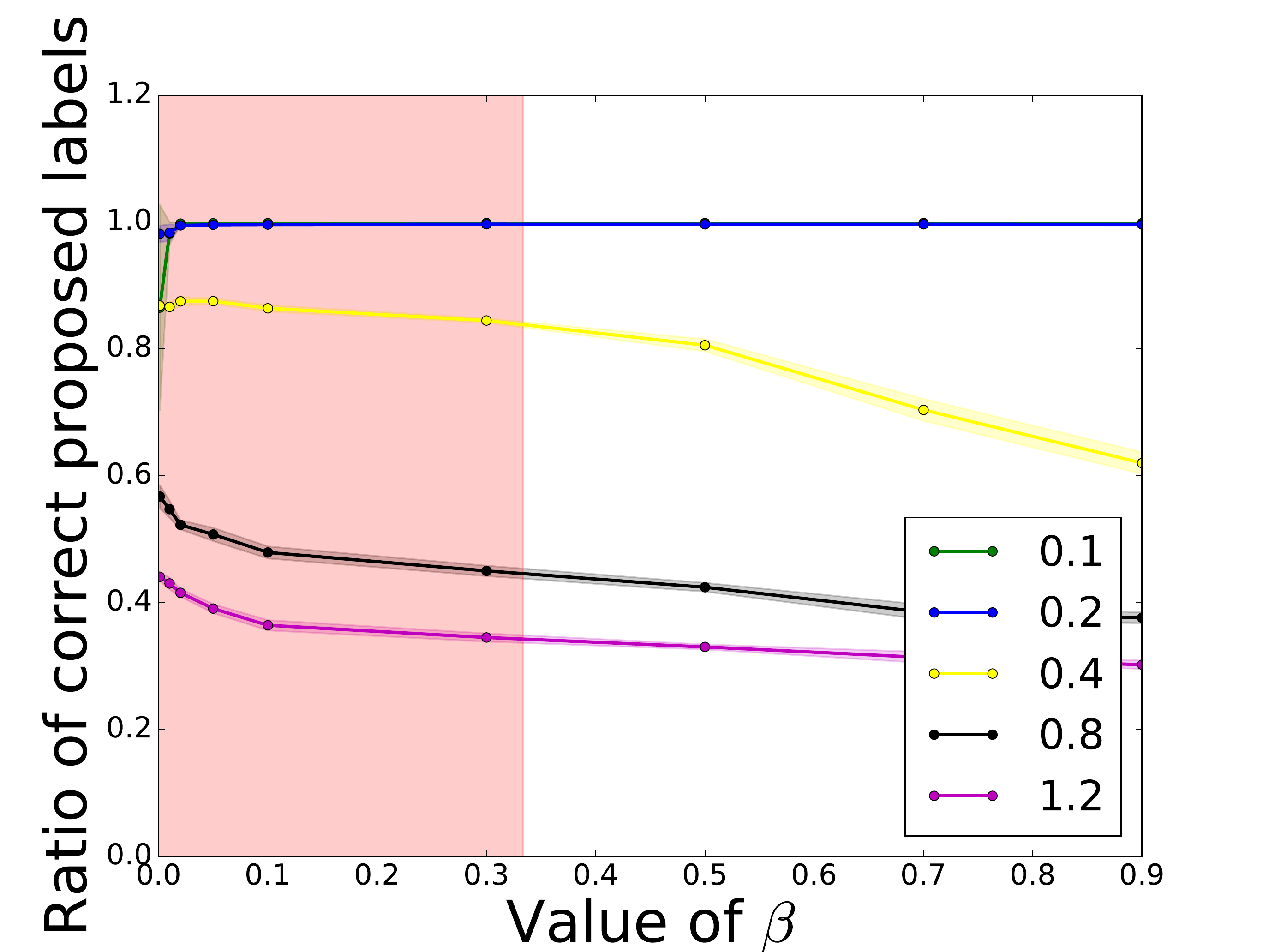}
\caption{Relevance between amount of noise $\sigma$ and $\beta$ on artificial dataset. Each legend corresponds to the amount of noise. Solid lines show the mean accuracy of ten trials, and the shaded areas around each plot show the standard deviation. The red shaded area shows the range of $\beta$ that can guarantee the convergence in Theorem 2.}
\label{figure0}
\end{figure}

\begin{figure*}[!t]
\begin{tabular}{ccc}
 \begin{minipage}[t]{0.3\textwidth}
  \centering
   \includegraphics[keepaspectratio, scale=0.17]{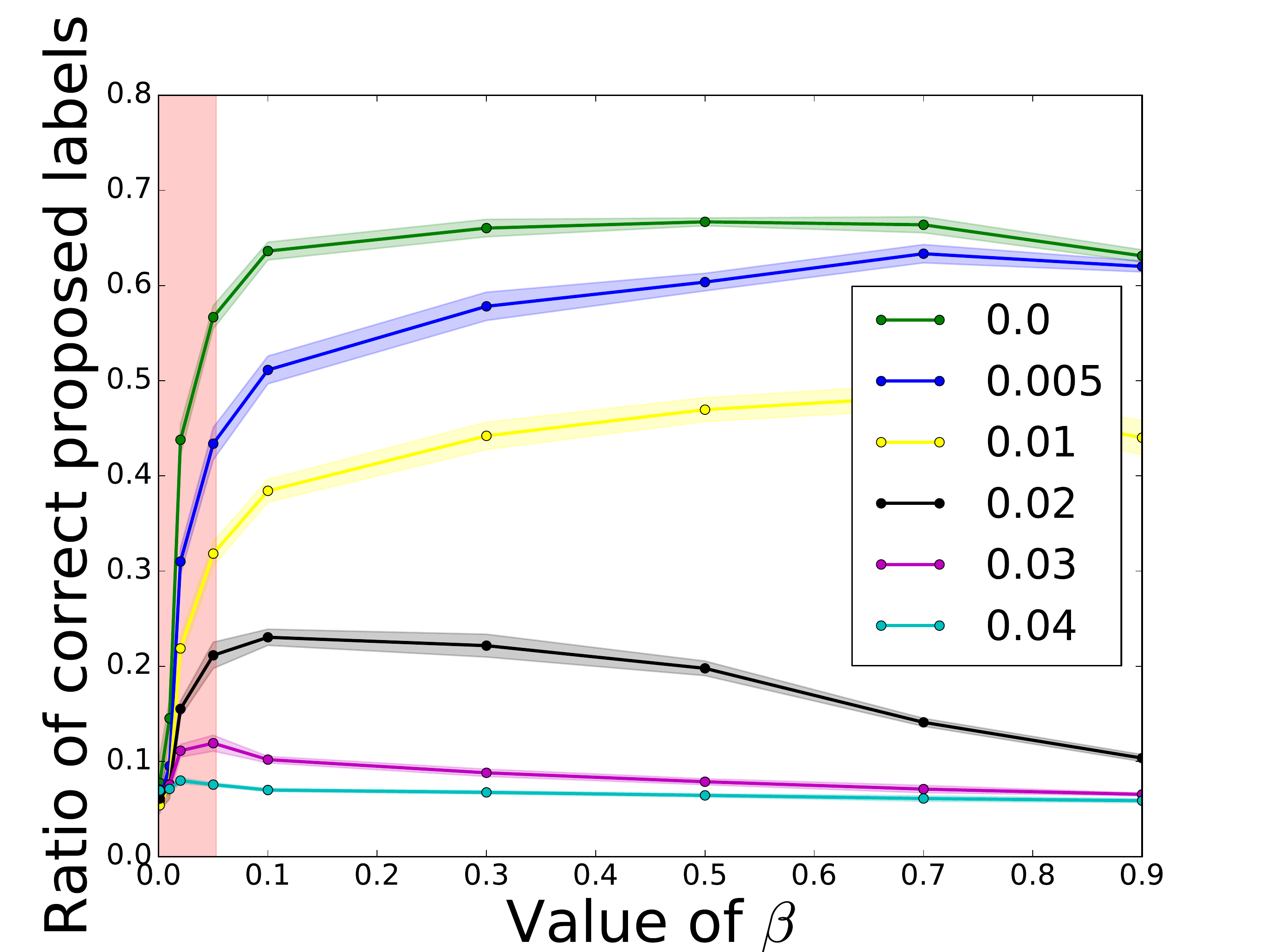} \\
   \subcaption{20News}
 \end{minipage} &
 \begin{minipage}[t]{0.3\textwidth}
  \centering
   \includegraphics[keepaspectratio, scale=0.17]{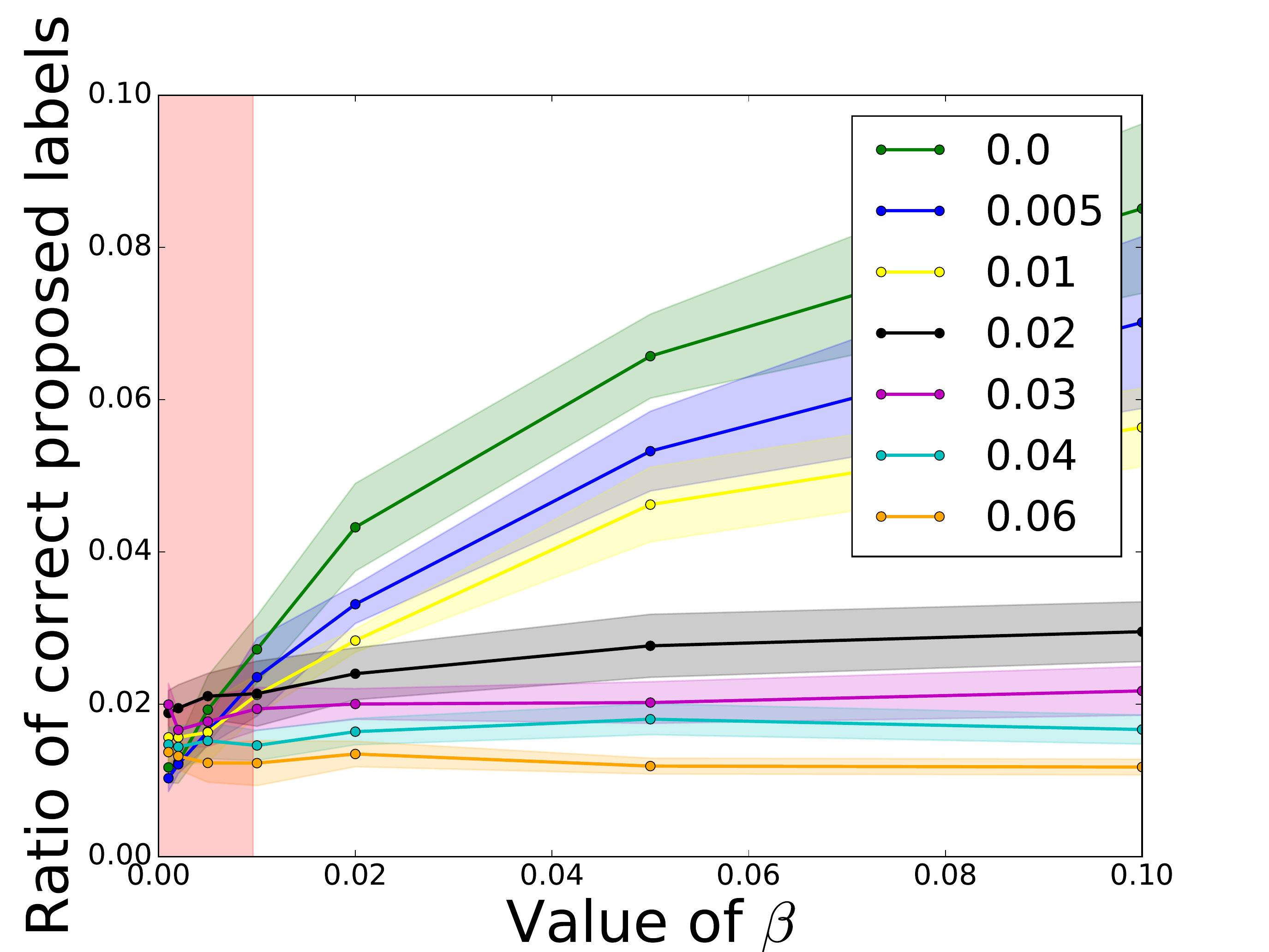} \\
   \subcaption{Sector}
 \end{minipage} & 
 \begin{minipage}[t]{0.3\textwidth}
  \centering
   \includegraphics[keepaspectratio, scale=0.17]{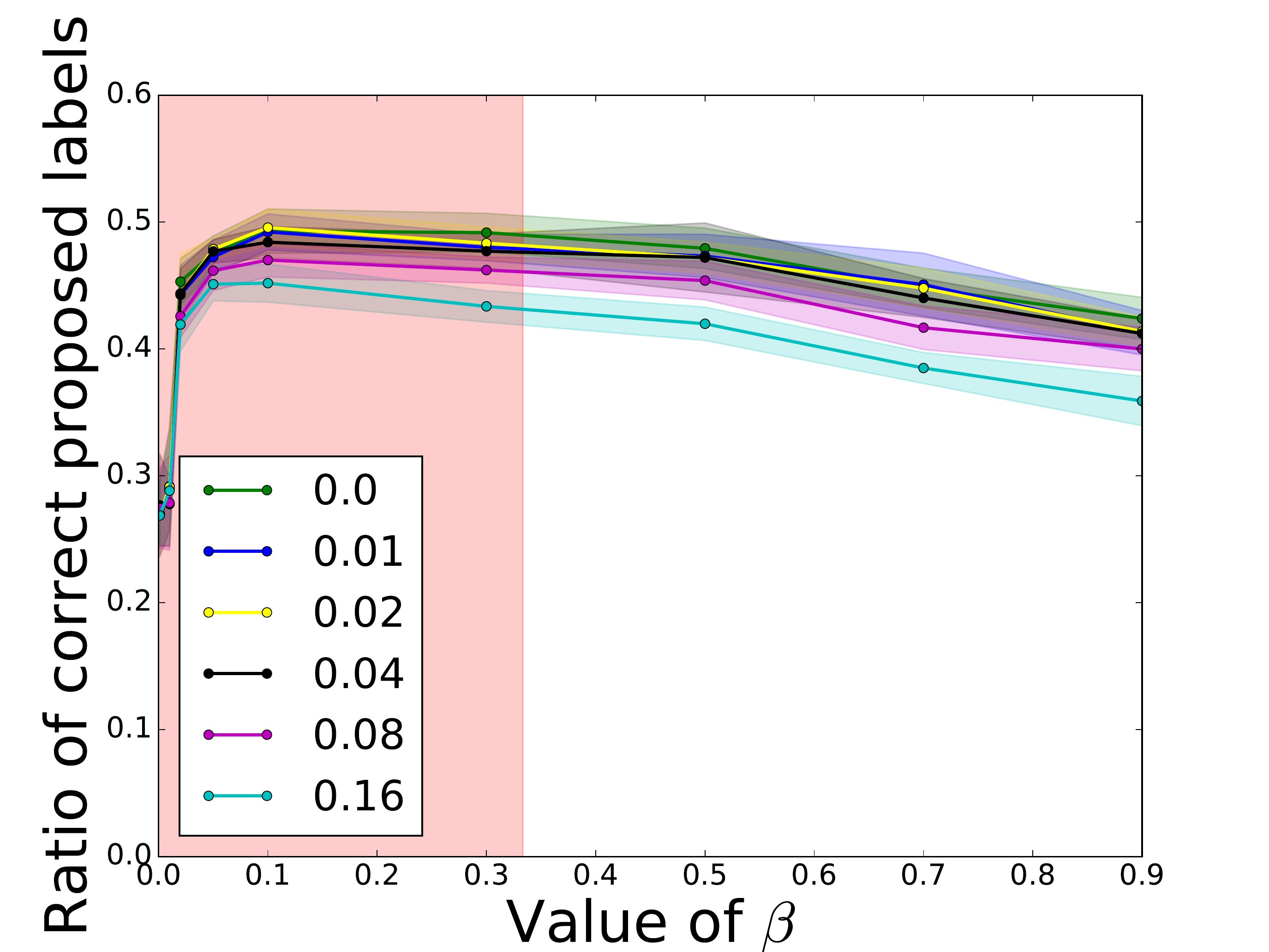}\\
   \subcaption{Vehicle}
 \end{minipage} \\
 \begin{minipage}[t]{0.3\textwidth}
  \centering
   \includegraphics[keepaspectratio, scale=0.17]{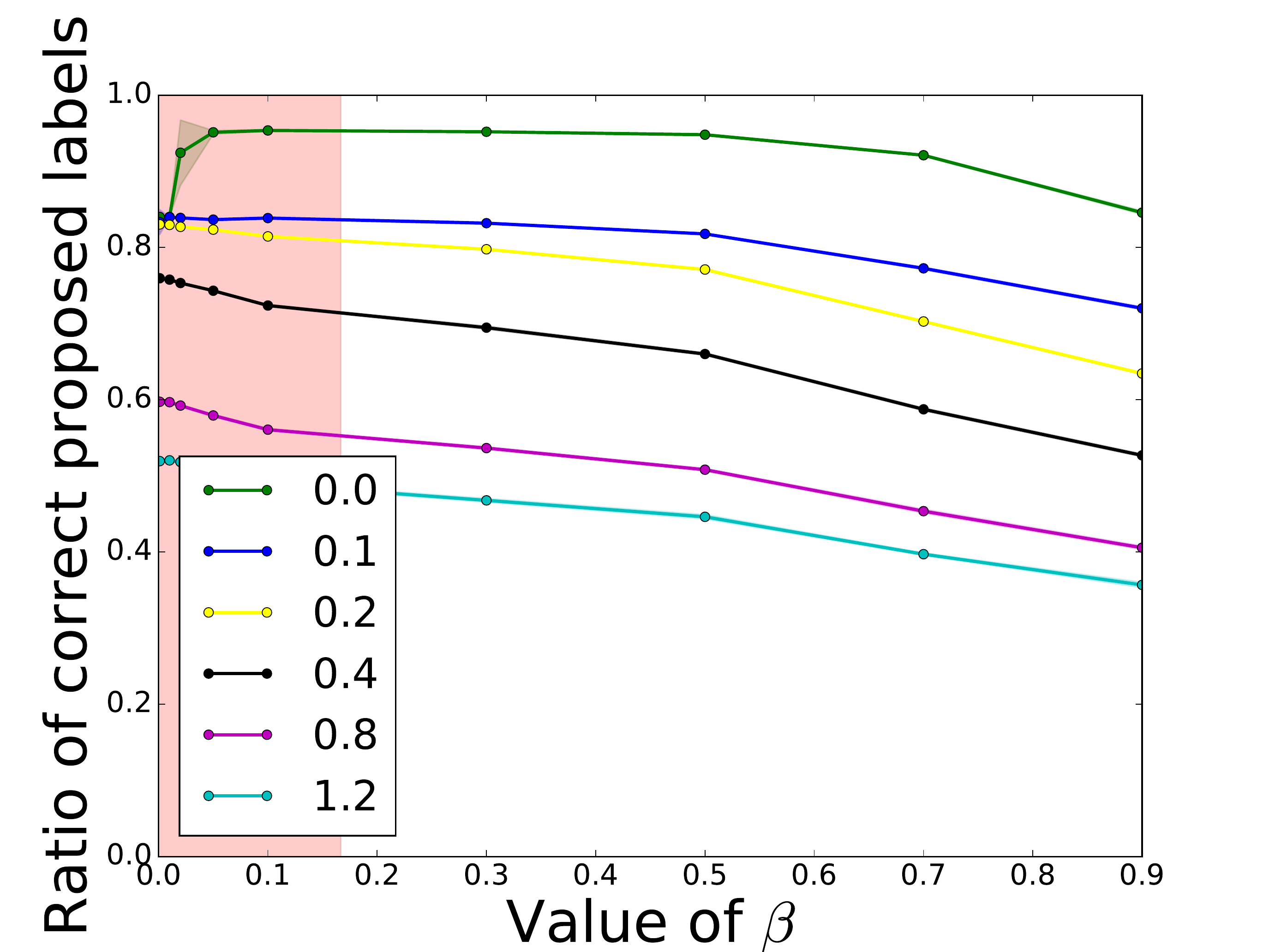} \\
   \subcaption{Shuttle}
 \end{minipage} &
 \begin{minipage}[t]{0.3\textwidth}
  \centering
   \includegraphics[keepaspectratio, scale=0.17]{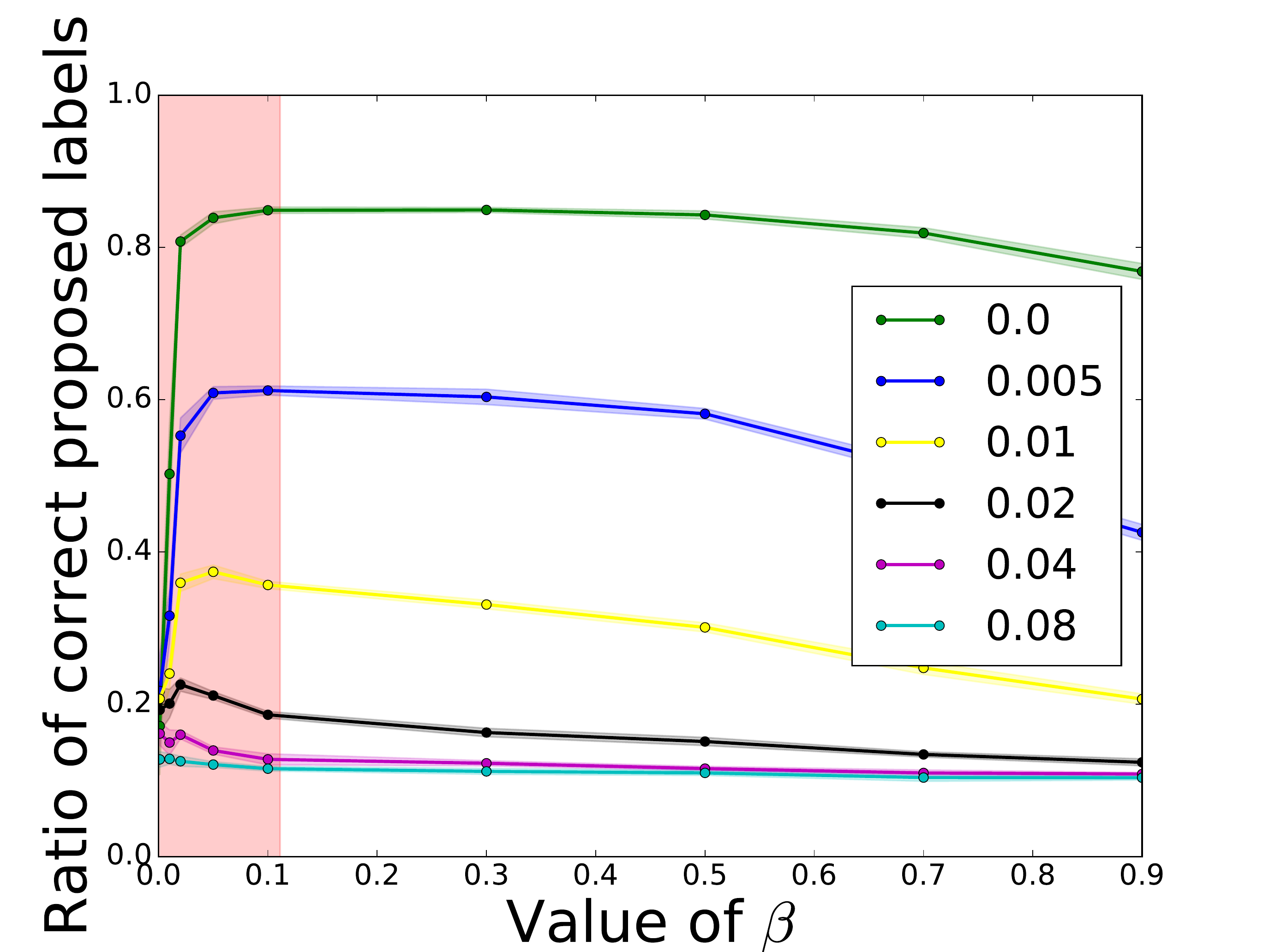} \\
   \subcaption{Usps}
 \end{minipage} & 
 \begin{minipage}[t]{0.3\textwidth}
  \centering
   \includegraphics[keepaspectratio, scale=0.17]{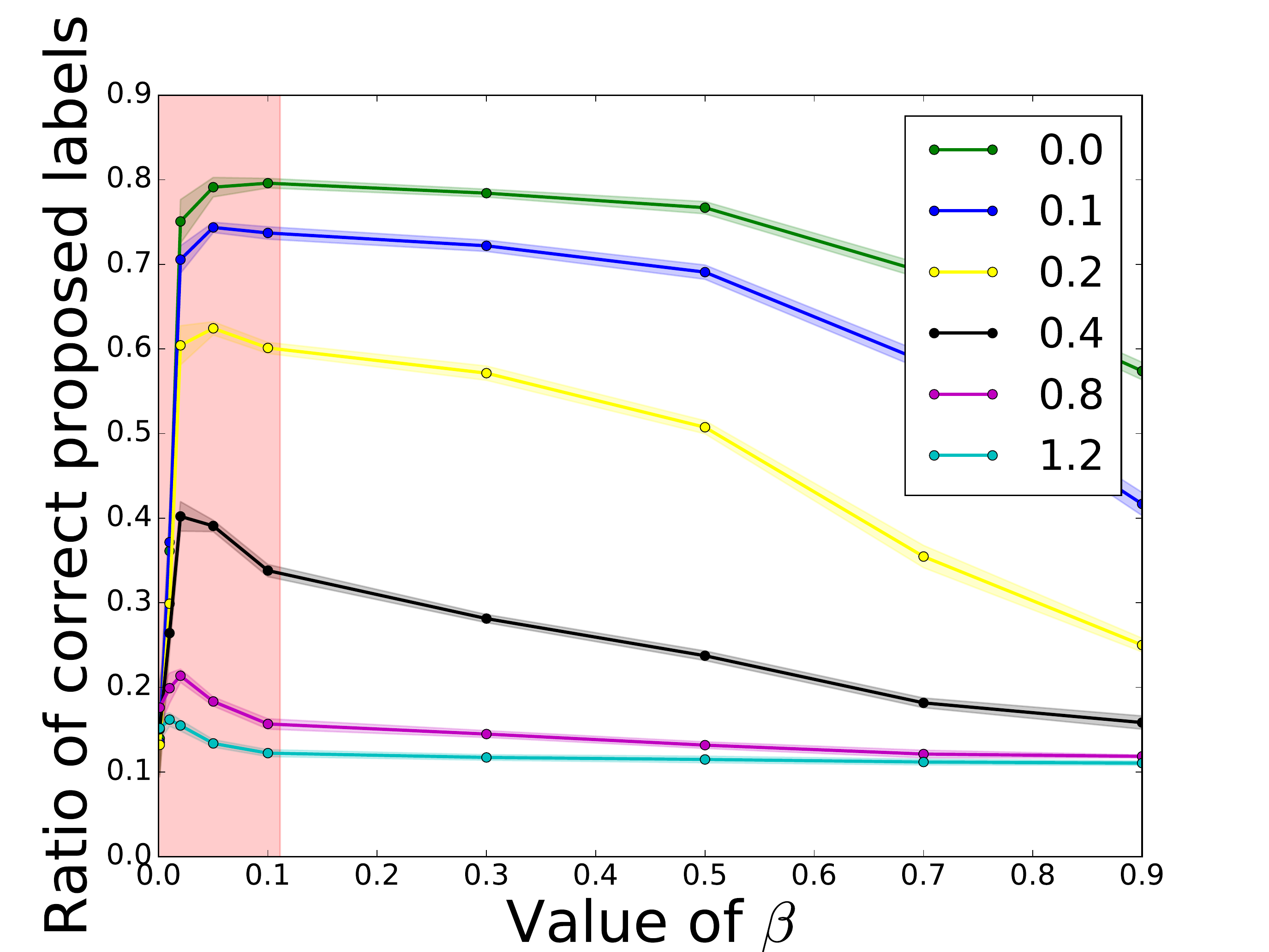} \\
   \subcaption{Pendigits}
 \end{minipage} \\
\end{tabular}
\caption{Relevance between amount of noise $\sigma$ and $\beta$ on real world datasets. Each legend corresponds to the amount of noise. Solid lines show the mean accuracy of ten trials, and the shaded areas around each plot show the standard deviation. The red shaded area shows the range of $\beta$ that can guarantee the convergence in Theorem 2.}
\label{figure1}
\end{figure*}

In this section, we demonstrate the experimental performance of proposed method, CSPA.
\vskip 0.1cm
\textbf{Datasets:} We used the following benchmark datasets: 20News, Sector, Vehicle, Shuttle,  Usps, Pendigits, Satimage, MNIST, Letter, Segment, Vowel, and Sensorless. The properties of these data are summarized in Table \ref{table1} and \ref{table2} in Appendix B in the supplementary material. All except MNIST can be downloaded from the {\em LIBSVM} \citep{libsvm} \footnote{\url{https://www.csie.ntu.edu.tw/~cjlin/libsvm/}}, and MNIST can be downloaded from the website of Sam Roweis \footnote{\url{http://cs.nyu.edu/~roweis/data.html.}}.  All instances were used for Segment and Vehicle, and training instances were used for the others. Normalization was applied to each feature vector if its norm is not one. 

\vskip 0.1cm
\textbf{Metrics:} In the partial feedback setting, the goal is to propose as many correct labels as possible while training. Therefore, following the existing research \citep{Banditron, Confidit}, we did not use the test accuracy as a metric, but instead evaluated the algorithms with the ratio of correct proposed labels while training.

\subsection{Relevance between hyperparameter $\beta$ and noisy data}
\label{rel}

First, we investigated the relevance between $\beta$ in the algorithm and how noisy the data is. We used a simple artificial data and real-world datasets, in particular, 20News, Sector, Vehicle, Shuttle, Usps and Pendigits. 

The artificial data had two dimensions and four classes. We generated $1,000$ samples for class $i$ from $\mathcal{N}\left(\mathbf{c}_i, \left(
    \begin{array}{cc}
      \sigma & 0 \\
      0 & \sigma \\
\end{array} \right)\right)$, where $\mathbf{c}_1 = (1, 1), \mathbf{c}_2 = (1, -1), \mathbf{c}_3 = (-1, 1), \mathbf{c}_4 = (-1, -1)$ and $\sigma$ is a positive real number. For real world datasets, we added a Gaussian noise of mean zero and standard deviation $\sigma$ to each feature. We chose the amount of noise so as to make it easy to see the degradation of accuracy for each dataset.

The results are shown in Figure \ref{figure0} and Figure \ref{figure1}. Off course, the larger the noise is, the lower the accuracy becomes, but smaller $\beta$ is robust to the degradation of accuracy. In particular, it is robust in the range satisfying $0 < \beta < \frac{1}{K-1}$, which can guarantee convergence in Theorem 2.

\subsection{Comparison with other methods}
\textbf{Algorithms:} We compared CSPA with the Banditron \citep{Banditron}, Confidit \citep{Confidit} and BPA \citep{BPA} algorithms in the partial feedback setting. Note that we implemented the Confidit algorithm in accordance with the experiments in \citet{Confidit}. CSPA and the other three algorithms require $O(Kd)$ memory for parameters, where $K$ is the number of classes and $d$ is the dimension of the feature vectors. The computational complexity per iteration is $O(Kd)$ for all the algorithms.

\subsubsection{Linear function case}
\begin{figure*}[t]
\begin{tabular}{ccc}
 \begin{minipage}[t]{0.3\textwidth}
  \centering
   \includegraphics[keepaspectratio, scale=0.17]{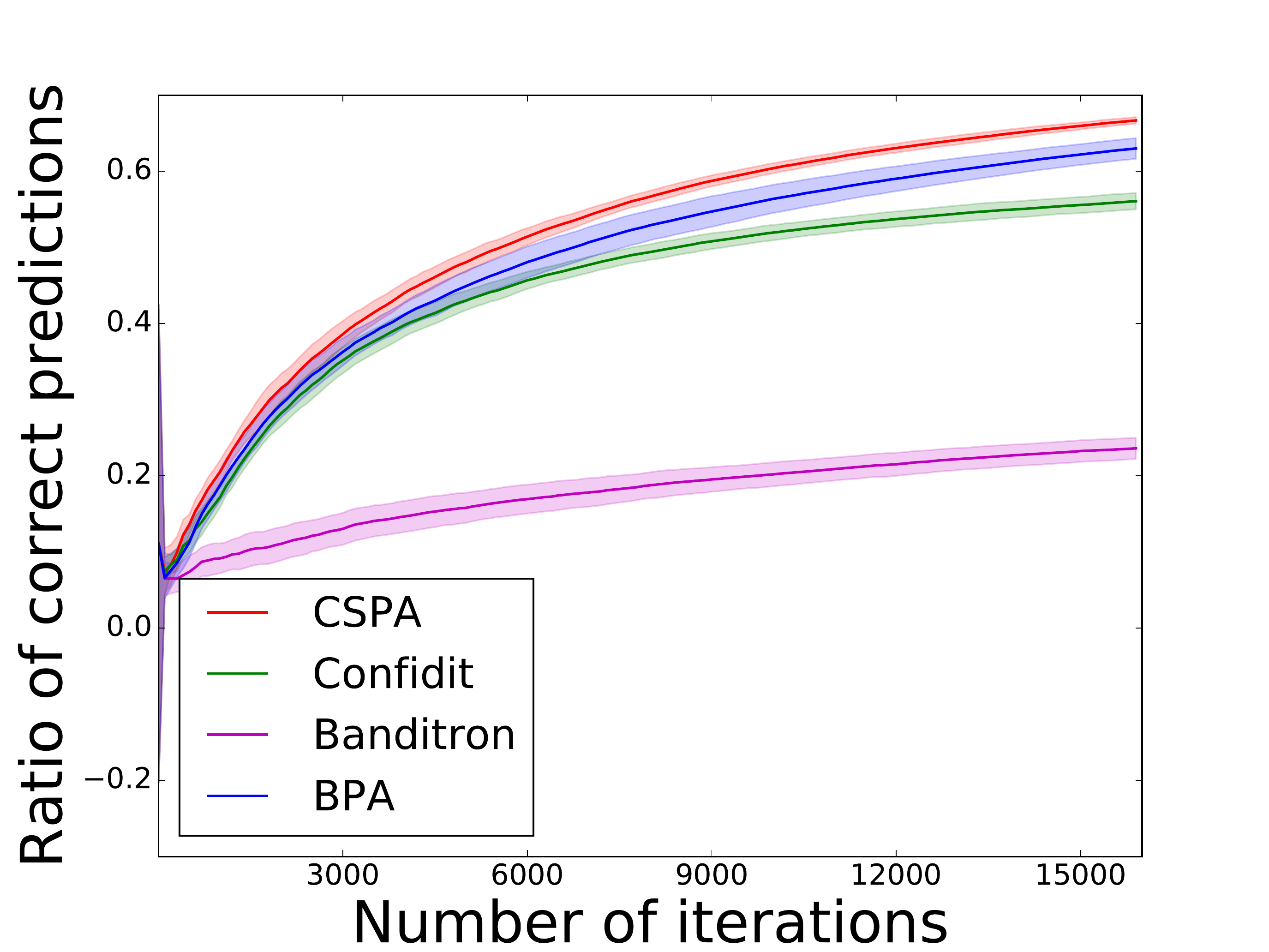} \\
   \subcaption{20News}
 \end{minipage} &
 \begin{minipage}[t]{0.3\textwidth}
  \centering
   \includegraphics[keepaspectratio, scale=0.17]{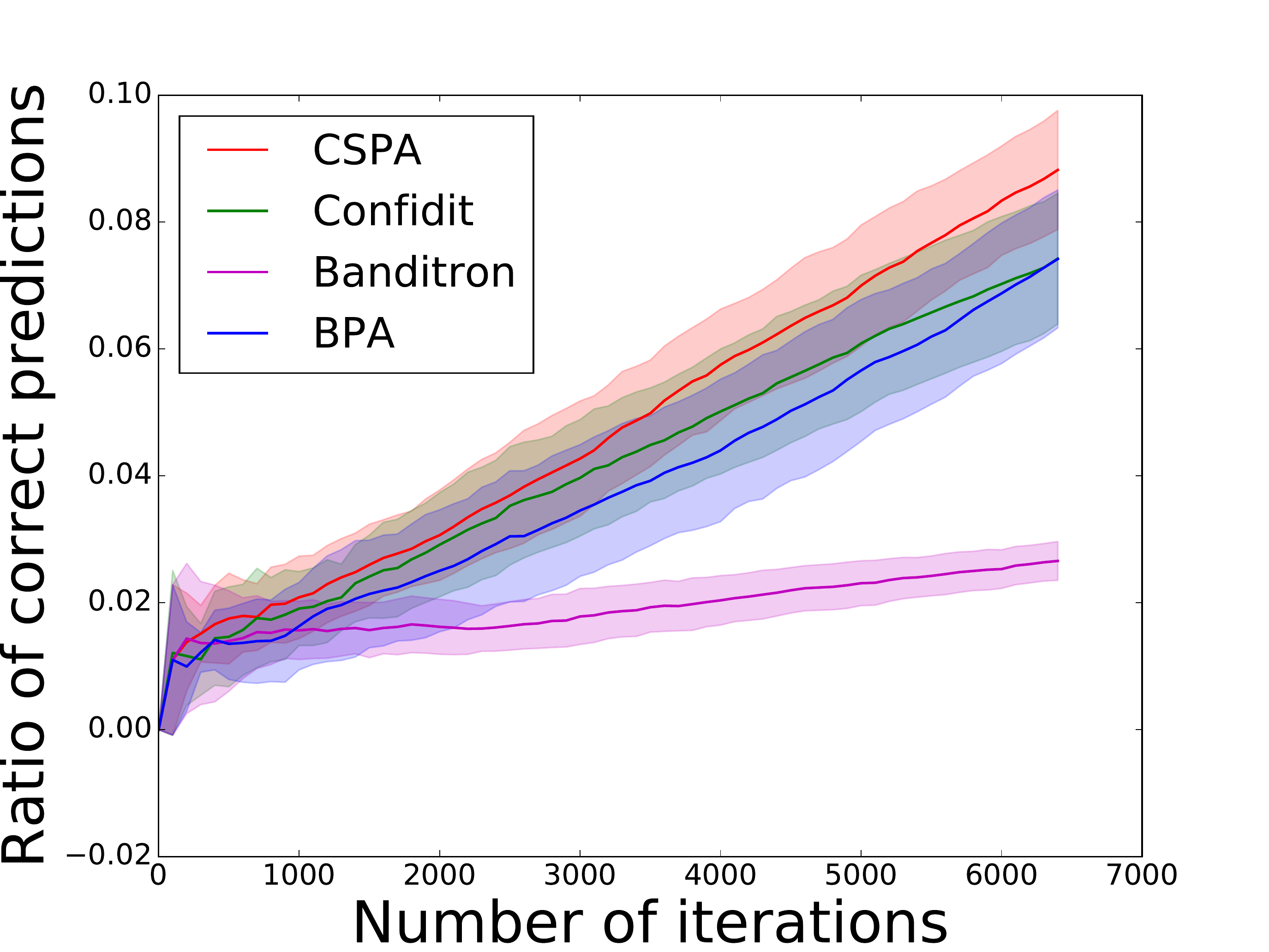} \\
   \subcaption{Sector}
 \end{minipage} & 
 \begin{minipage}[t]{0.3\textwidth}
  \centering
   \includegraphics[keepaspectratio, scale=0.17]{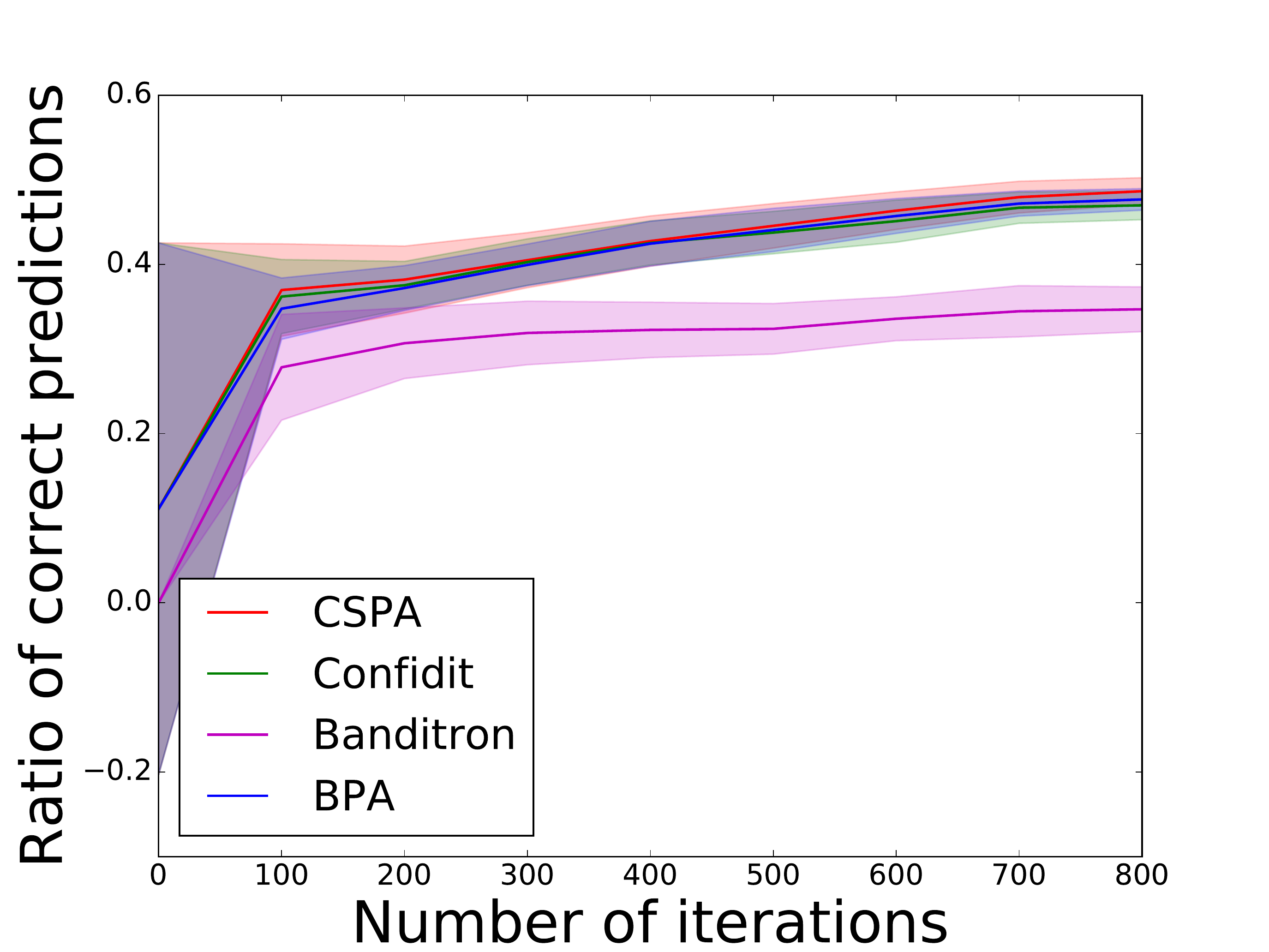}\\
   \subcaption{Vehicle}
 \end{minipage} \\
 \begin{minipage}[t]{0.3\textwidth}
  \centering
   \includegraphics[keepaspectratio, scale=0.17]{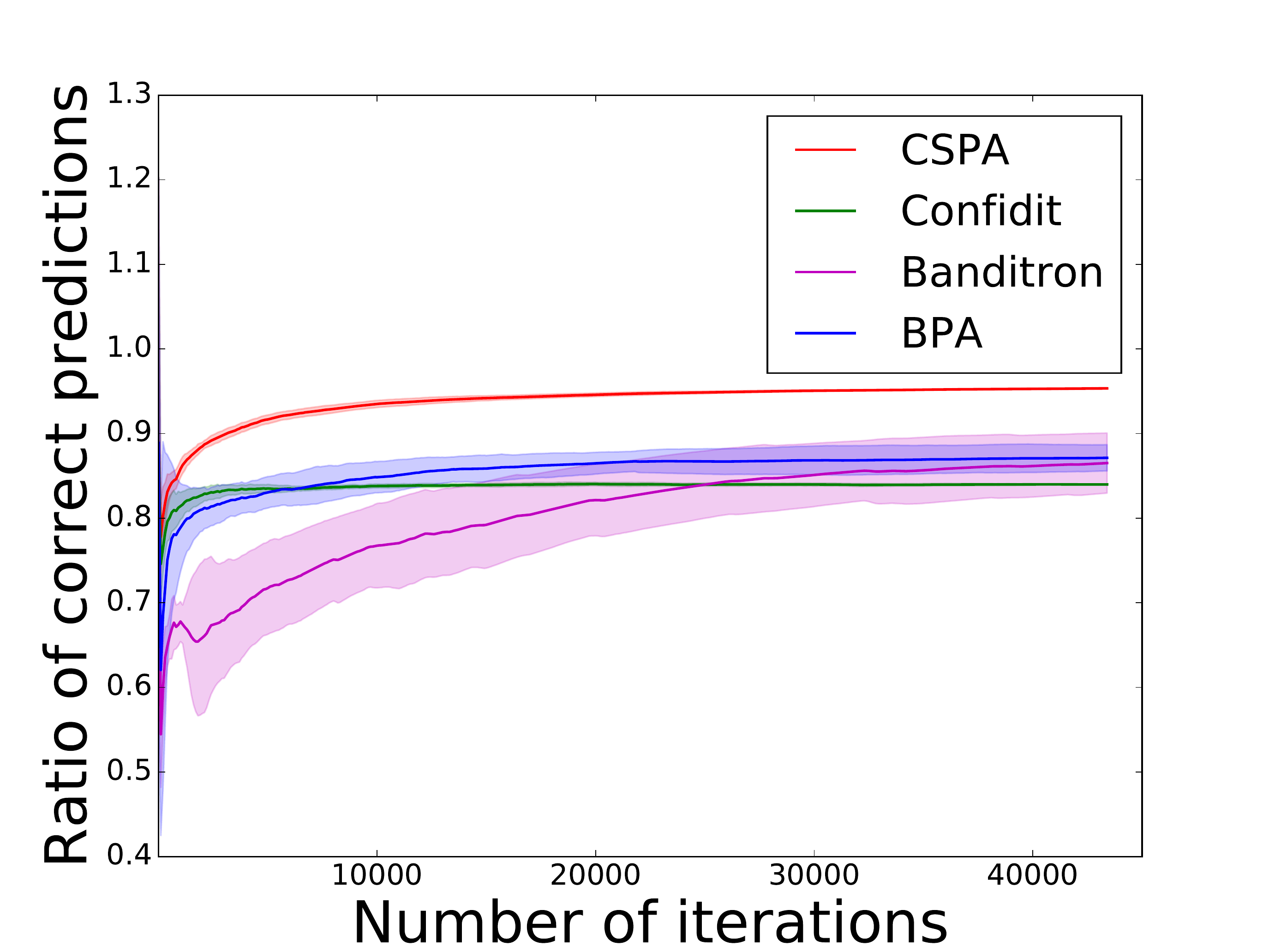} \\
   \subcaption{Shuttle}
 \end{minipage} &
 \begin{minipage}[t]{0.3\textwidth}
  \centering
   \includegraphics[keepaspectratio, scale=0.17]{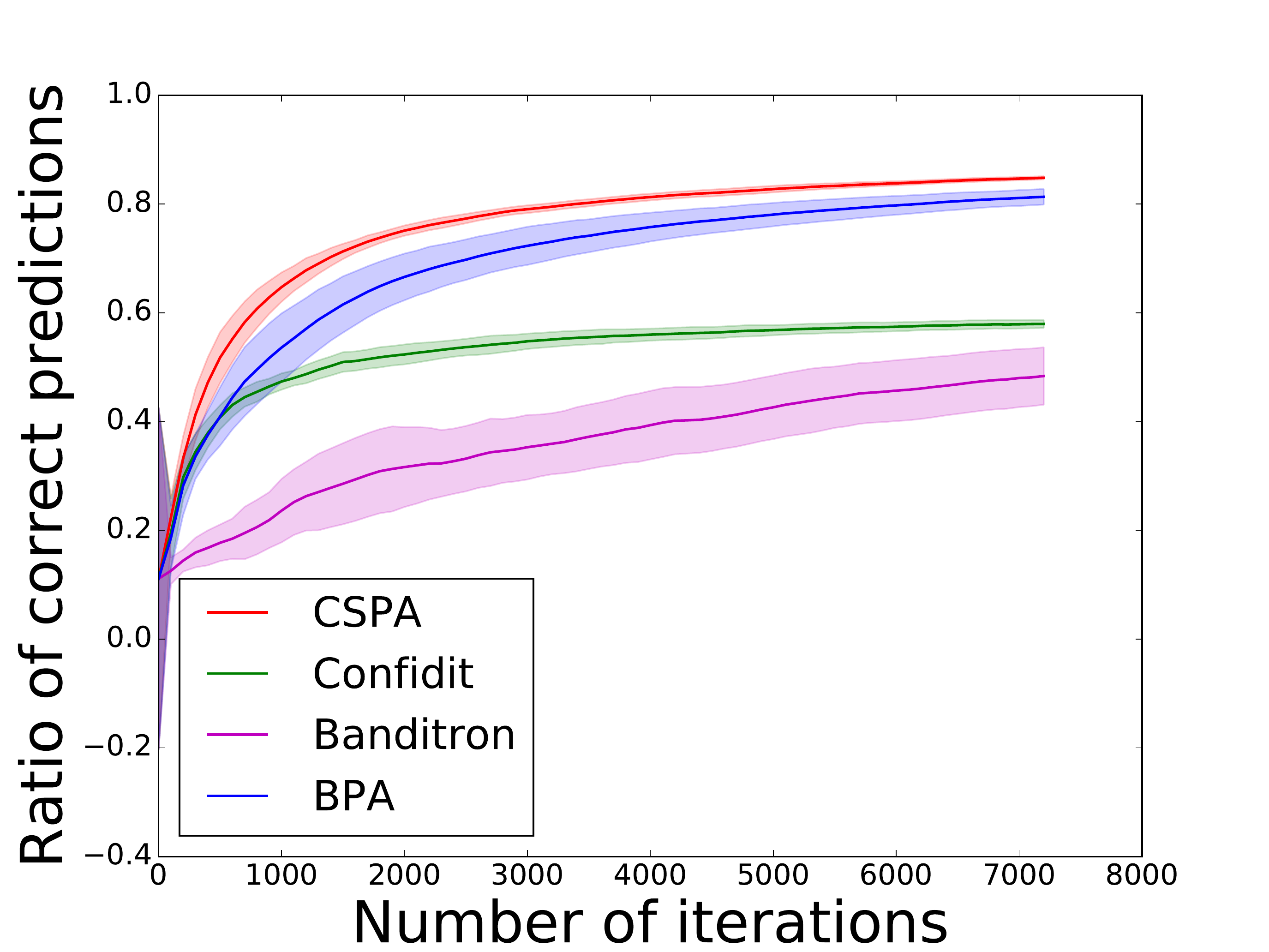} \\
   \subcaption{Usps}
 \end{minipage} & 
 \begin{minipage}[t]{0.3\textwidth}
  \centering
   \includegraphics[keepaspectratio, scale=0.17]{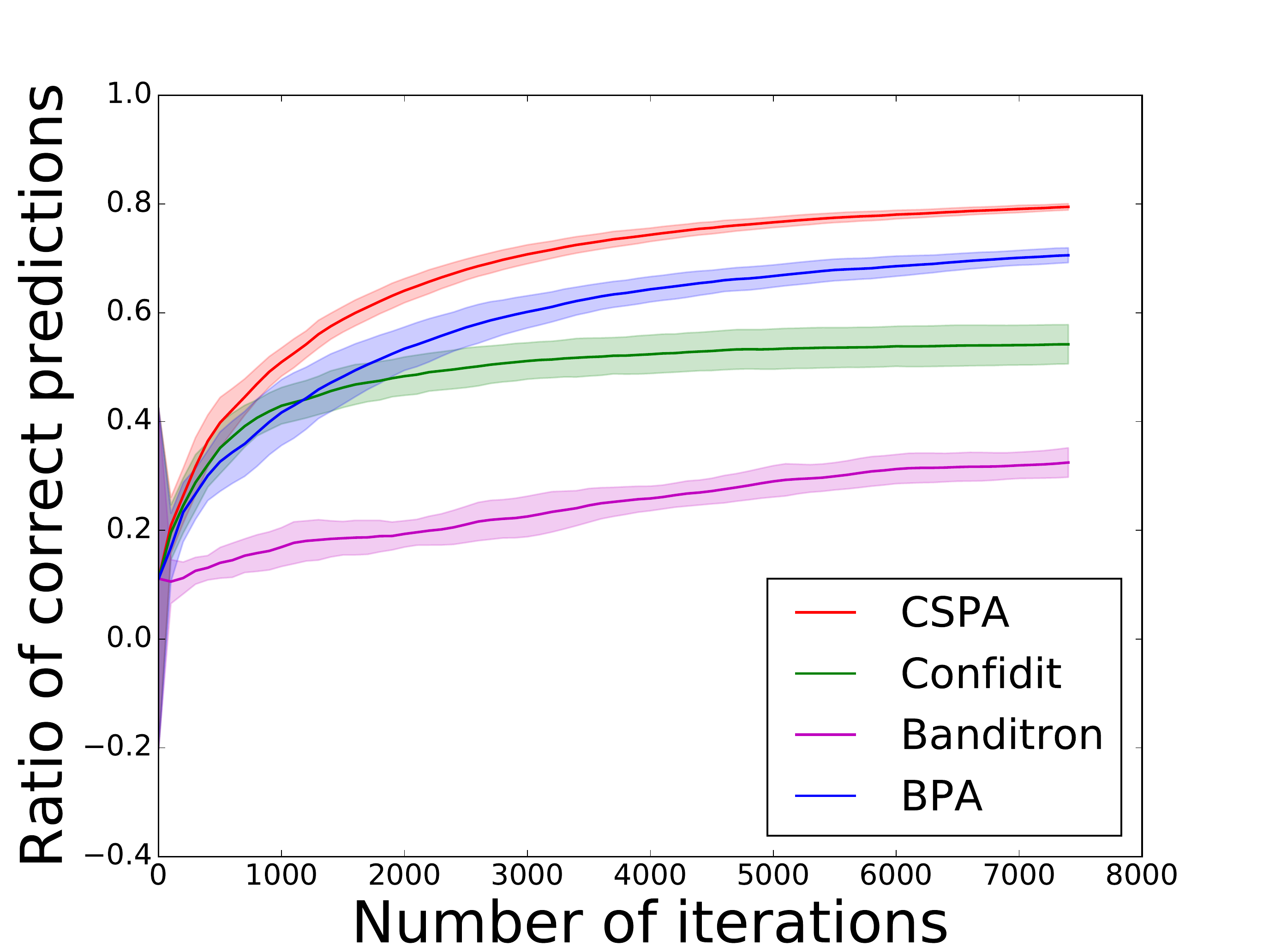} \\
   \subcaption{Pendigits}
 \end{minipage} \\
\end{tabular}
\caption{Ratios of correct predictions in partial feedback setting in the linear function case. Solid lines show the mean of ten trials, and the shaded areas around each plot show the standard deviation.}
\label{figure2}
\end{figure*}

\renewcommand{\arraystretch}{1.1}

\textbf{Parameter selection:} For the hyperparameter selection, following the experiment described in \citet{Banditron}, we compared the ratio of correct \emph{proposed} labels with ten different parameters. We compared the candidates $\{0.001, 0.01, 0.025, \allowbreak 0.05, 0.1, 0.2, 0.3, 0.4, 0.5, 0.6\}$ for $\gamma$ in Banditron and BPA, $\{10^{-4}, 10^{-3}, 10^{-2}, 10^{-1}, \allowbreak 10^0, 10^1, 10^2, 10^3, 10^4, 10^5\}$ for $\eta$ in Confidit and $\{0.1, 0.2, 0.3, 0.4, \allowbreak 0.5, 0.6, 0.7, 0.8, 0.9, \frac{1}{2(K-1)}\}$ for $\beta$ in CSPA, and chose the best hyperparameter, i.e., the one which attained the best ratio of the correct proposed labels for each algorithm. Here, $\beta = \frac{1}{2(K-1)}$ in the CSPA algorithm corresponds to the case $\alpha = \frac{1}{2}$ in Theorem \ref{main_theorem}.

\vskip 0.1cm
\textbf{Results:}  We evaluated the ratios of correct proposed labels in ten different runs for the four algorithms and took the average of every $100$ rounds. The results are shown in Figure \ref{figure2}. Figure \ref{figure2} shows the transitions of ratios of correct labels for the different datasets. CSPA outperforms the other three algorithms on five datasets and performed competitively on all datasets. In addition, as you can see from the shaded areas around each plot in Figure \ref{figure2}, CSPA is more stable than the others. The final results are shown in Table \ref{table1} in the supplemental material.

\subsubsection{Nonlinear function case}
\begin{figure*}[t]
\begin{tabular}{ccc}
 \begin{minipage}[t]{0.3\textwidth}
  \centering
   \includegraphics[keepaspectratio, scale=0.17]{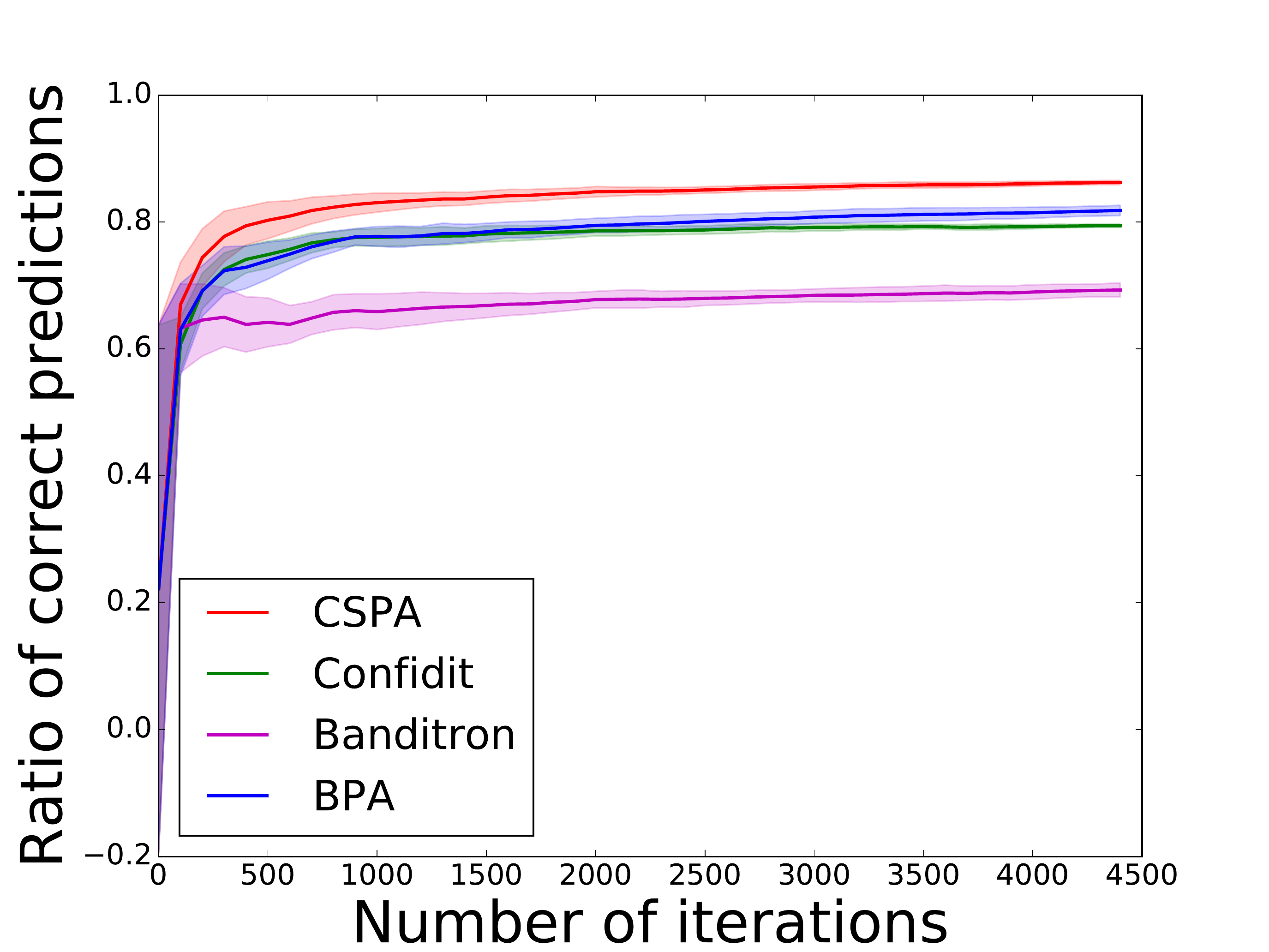} \\
   \subcaption{Satimage}
 \end{minipage} &
 \begin{minipage}[t]{0.3\textwidth}
  \centering
   \includegraphics[keepaspectratio, scale=0.17]{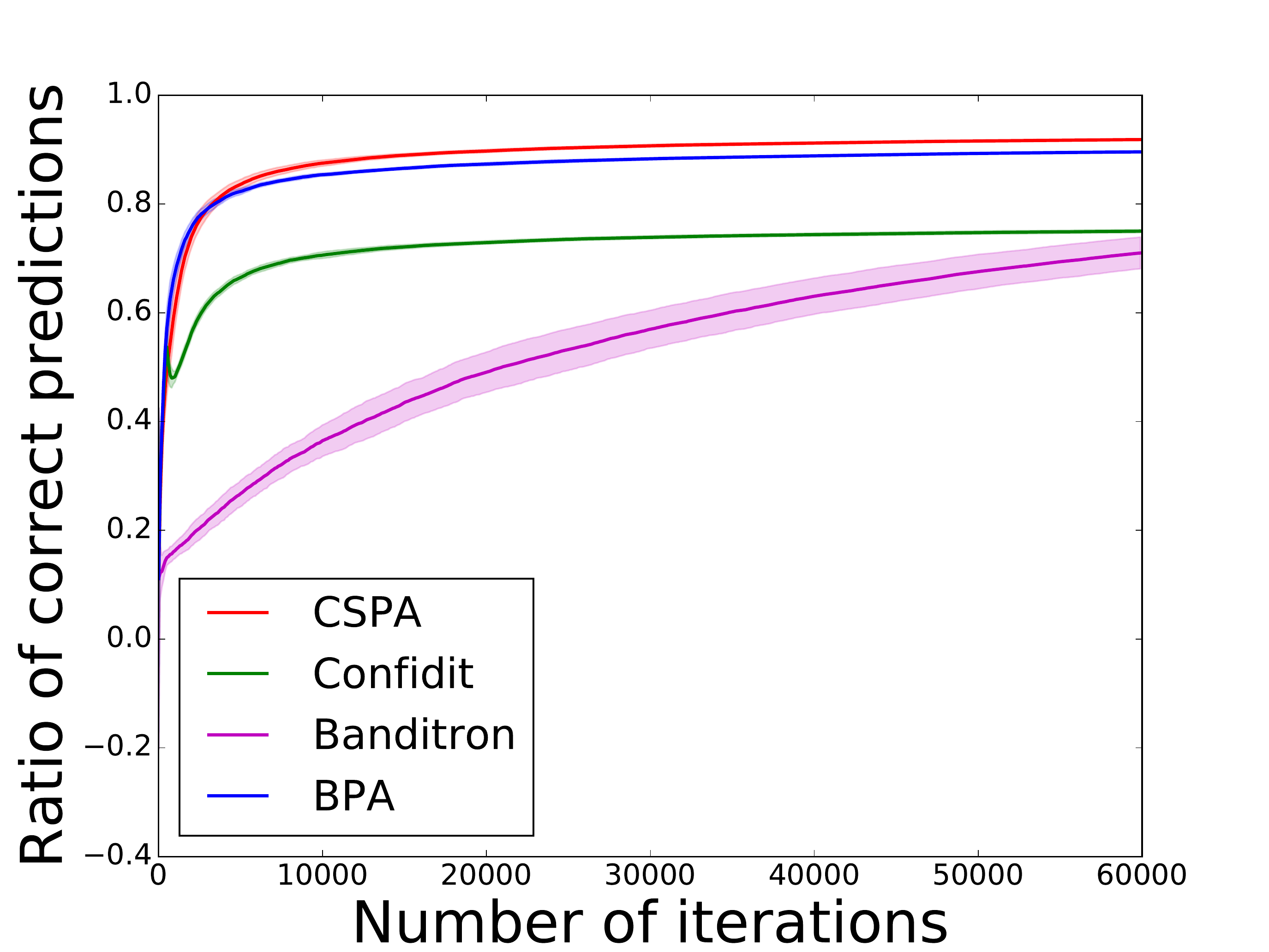} \\
   \subcaption{MNIST}
 \end{minipage} &
 \begin{minipage}[t]{0.3\textwidth}
  \centering
   \includegraphics[keepaspectratio, scale=0.17]{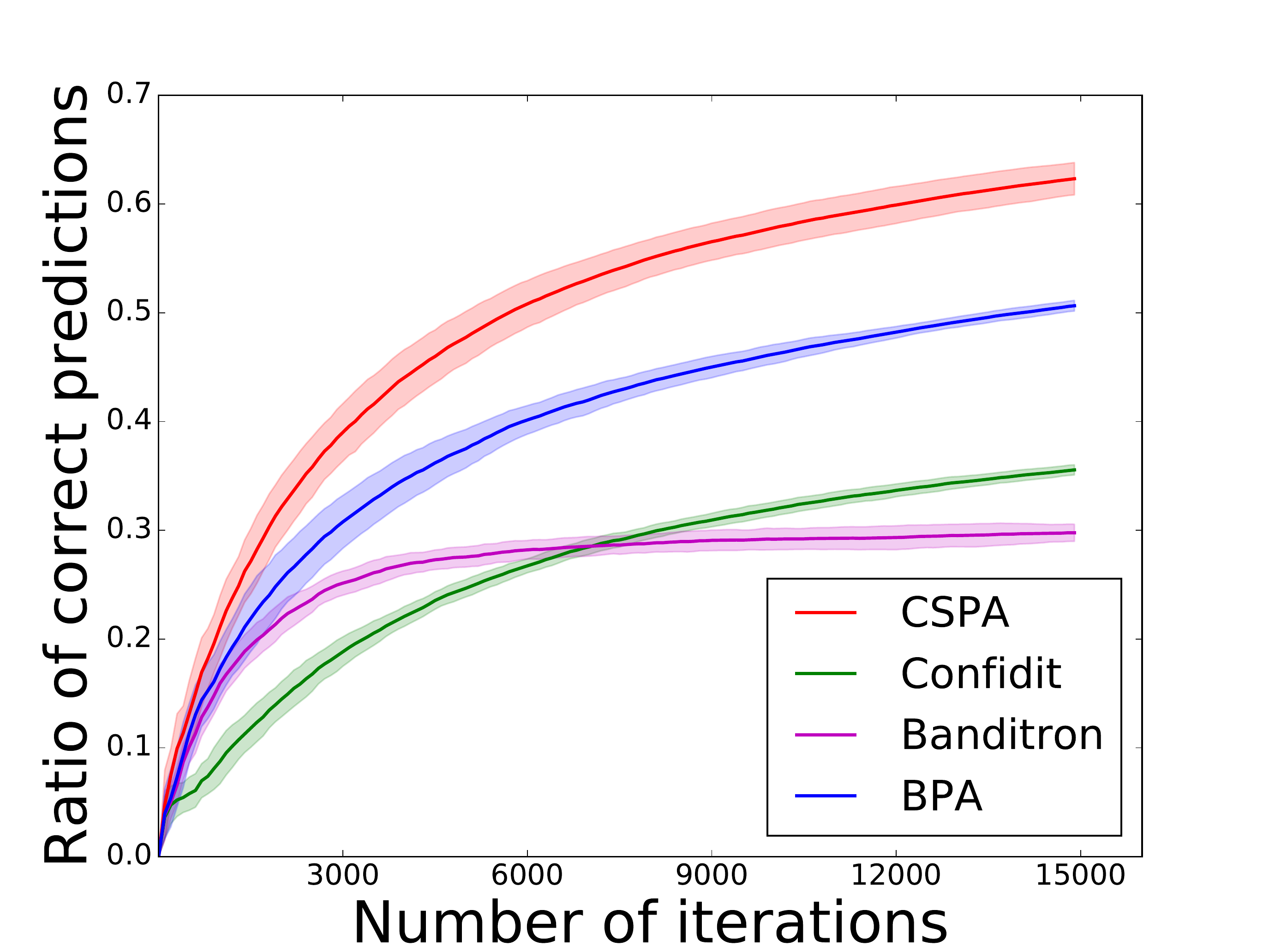} \\
   \subcaption{Letter}
 \end{minipage} \\
 \begin{minipage}[t]{0.3\textwidth}
  \centering
   \includegraphics[keepaspectratio, scale=0.17]{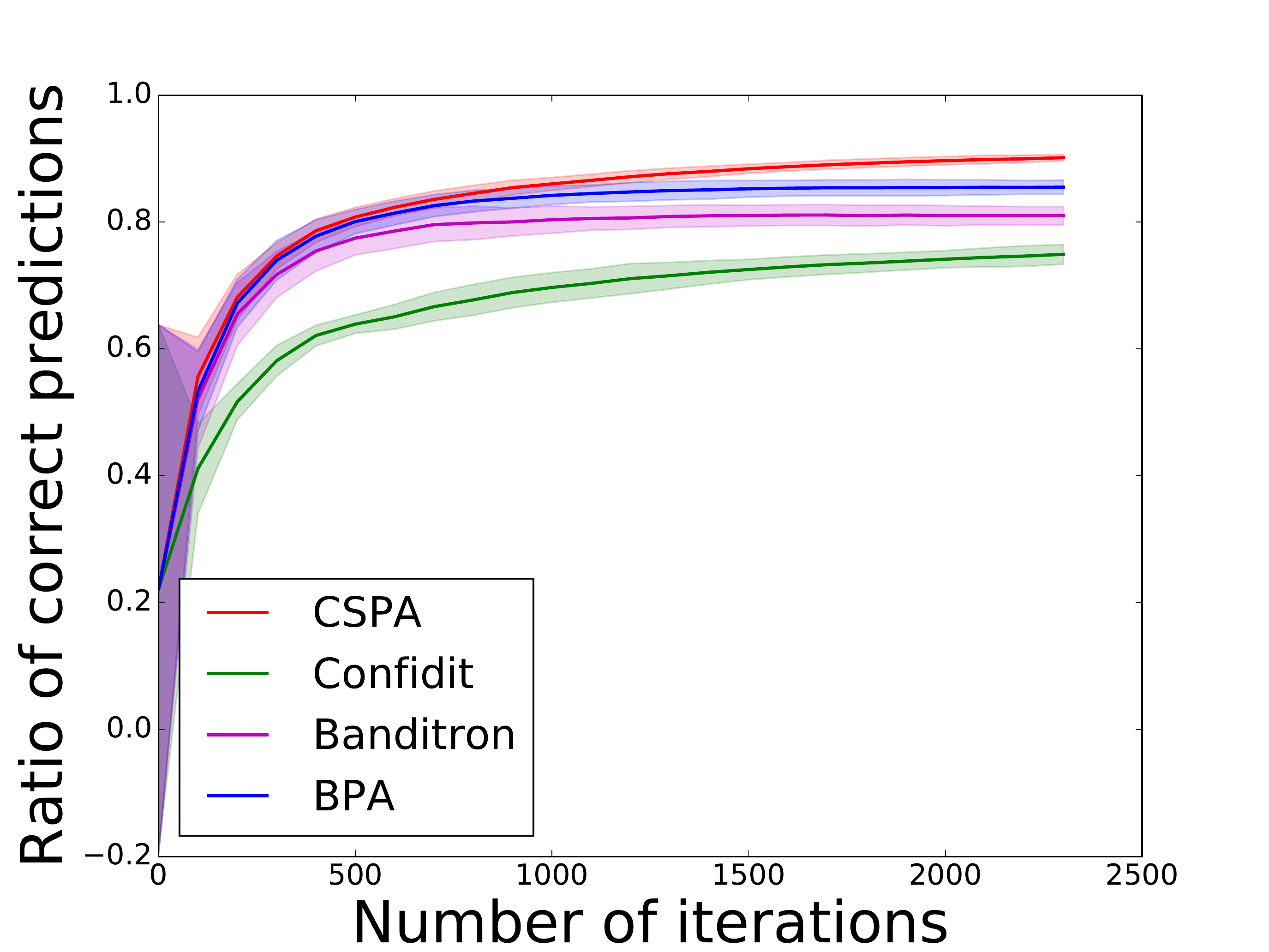} \\
   \subcaption{Segment}
 \end{minipage} & 
 \begin{minipage}[t]{0.3\textwidth}
  \centering
   \includegraphics[keepaspectratio, scale=0.17]{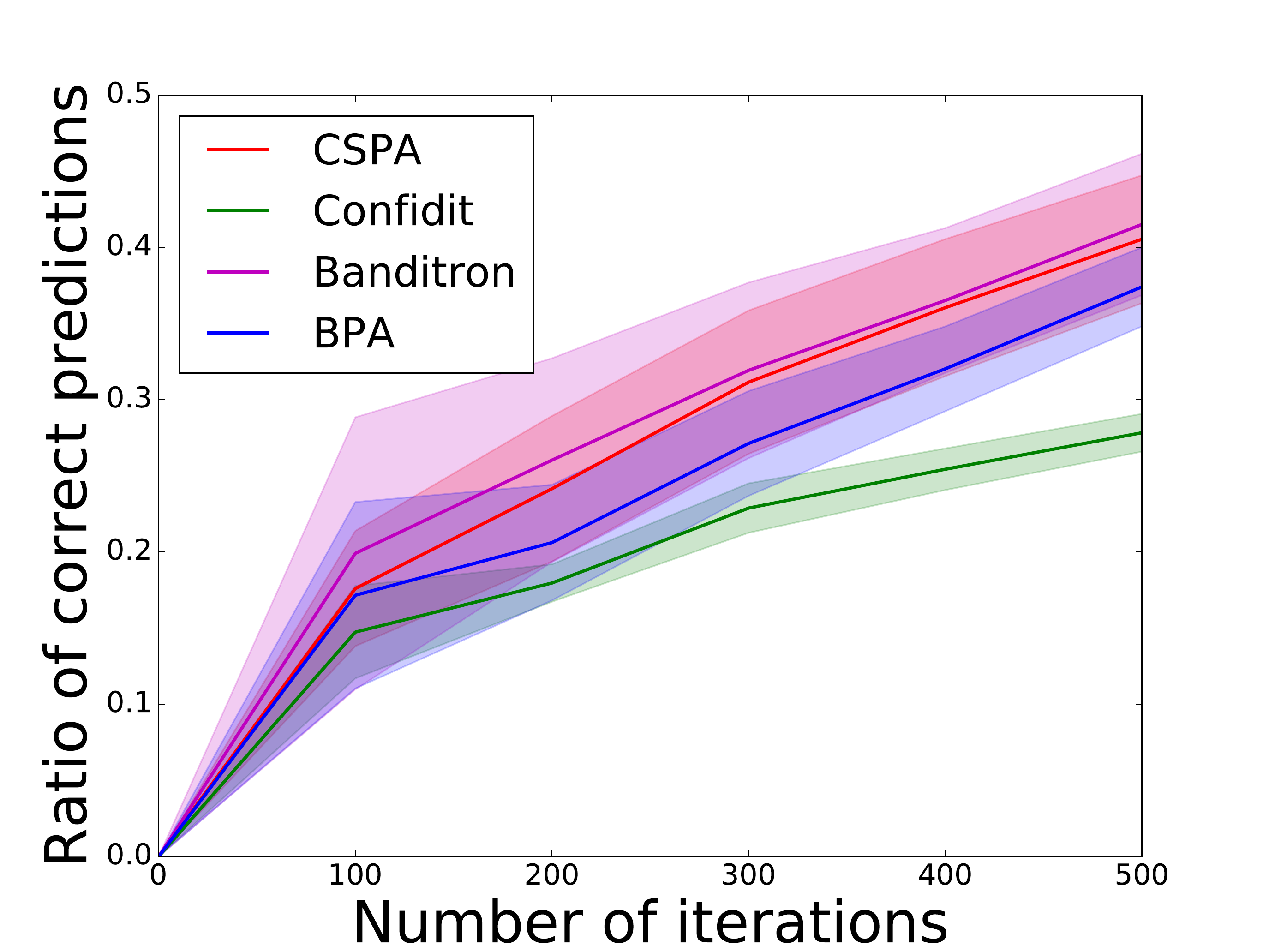} \\
   \subcaption{Vowel}
 \end{minipage} & 
 \begin{minipage}[t]{0.3\textwidth}
  \centering
   \includegraphics[keepaspectratio, scale=0.17]{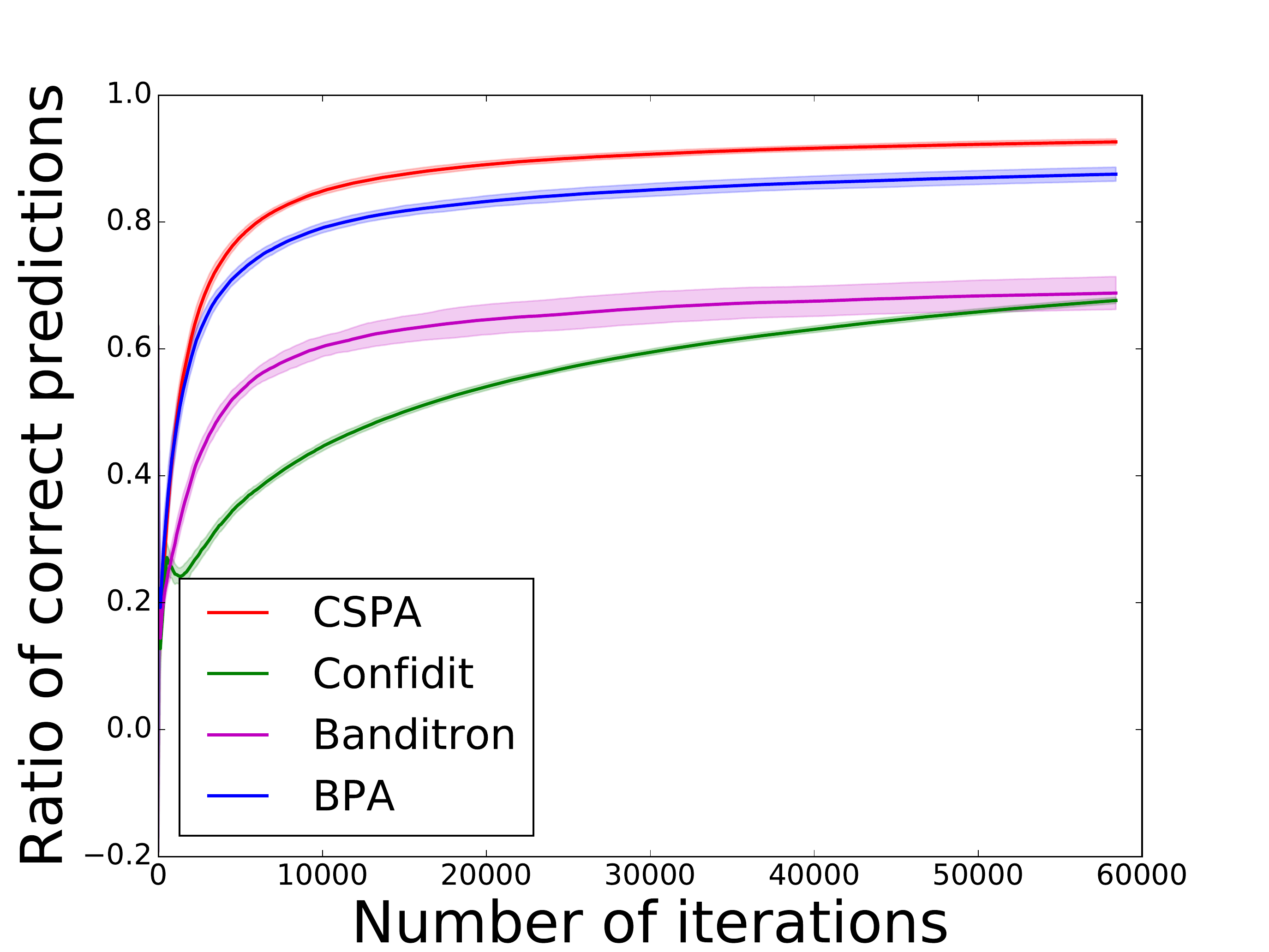} \\
   \subcaption{Sensorless}
 \end{minipage}
\end{tabular}
\caption{Ratios of correct predictions in partial feedback setting in nonlinear function case. Solid lines show the means of ten trials, and shaded areas show the standard deviation.}
\label{figure3}
\end{figure*}

To demonstrate the effectiveness of CSPA when the classification function is nonlinear, we experimentally compared CSPA with the other three algorithms using a nonlinear function. In order to make the classification function nonlinear, we used the Gaussian kernel $k({\bf x}, {\bf x}_i) = \exp\left(-\frac{\|{\bf x} - {\bf x}_i\|^2}{g}\right)$. Also, the first $700$ instances were used as a support set for kernels and we denote this set as $B$. That is, we used the following model:
\begin{equation}
f({\bf x}) = \argmax_{i \in \{1, \ldots, K\}} \sum_{j=1}^{700} {\bf w}_{i, j} k({\bf x}, {\bf x}_j),
\end{equation}
where ${\bf w}_{i, j}$ represents the $j$-th parameter of ${\bf w}_i$.

From the above, the four algorithms need $O(|B|(K+D))$ parameters and the computational complexity per iteration is $O(|B|KD)$.  We used only the first $528$ instances from the Vowel dataset as the support set because the number of instances in Vowel is less than $700$. We used the technique in \citet{Kernelmethod} to apply the kernel method to the Banditron and Confidit algorithm, which are based on the perceptron algorithm.

\vskip 0.1cm
\textbf{Parameter selection: }  We used a grid search to make the parameter selection. We compared candidates $\{0.01, 0.1, 1.0, 10.0, 100.0\}$ for $g$ of the Gaussian kernel in all algorithms, and $\{0.001, 0.025, 0.1, 0.3, 0.4, 0.6\}$ for $\gamma$ in Banditron and BPA, $\{10^{-4}, 10^{-2}, 10^0, 10^2, 10^4\}$ for $\eta$ in Confidit and $\{0.1, 0.3, 0.5, 0.7, \frac{1}{2(K-1)}\}$ for $\beta$ in CSPA. That is, all four algorithms selected the best pair of hyperparameters from 25 candidates.

\vskip 0.1cm
\textbf{Results: } Similarly to the linear function case, we evaluated the ratios of correct proposed labels in ten different runs of the four algorithms and took the average of every 100 rounds. The results are shown in Figure \ref{figure3}. When the classification function is nonlinear, CSPA also outperforms the other three algorithms on almost all datasets. The final results are shown in Table \ref{table2} in the supplemental material.

    \section{Discussions}
As shown in Section \ref{rel}, the range of the hyperparameter~$\beta$ that gives the convergence guarantee is robust to noisy data. We also showed that for clean data, choosing $\beta$ from this range is not necessarily better. Our theoretical analysis considered the adversarial case, so the algorithm should handle all the cases; as a result, it should behave more conservatively and $\beta$ should be set to a lower value. 

In terms of applications, there would be few cases where we should consider the adversarial case, so choosing $\beta$ outside the guaranteed range would give good empirical results. In addition, a theoretical analysis of less adversarial cases like \citet{Confidit} would be worth being considered.
    \section{Conclusion}
In this paper, we proposed CSPA, a novel online multiclass classification algorithm based on the prediction margin for the partial feedback setting. Our algorithm focused on the prediction margin and learning from complementary labels in the context of online classification. 
Our experiments showed that CSPA significantly outperformed other methods in the same setting. Furthermore, we provided a theoretical guarantee for CSPA through deriving a cumulative square loss bound, which is an upper bound of the number of mistakes. 

For another direction, \citet{mr} considered multi-label classification with partial feedback, where the correct labels of each instance are not necessarily one. This setting includes the multiclass classification case and can be applied to more applications, so extending our algorithm to this case would be a future work.
    \clearpage
    \section*{Acknowledgement}
IS was supported by JST CREST Grant Number JPMJCR17A1.
MS was supported by the International Research Center for Neurointelligence (WPI-IRCN) at The University of Tokyo Institutes for Advanced Study.

    \bibliographystyle{plainnat}

\begin{thebibliography}{20}
\providecommand{\natexlab}[1]{#1}
\providecommand{\url}[1]{\texttt{#1}}
\expandafter\ifx\csname urlstyle\endcsname\relax
  \providecommand{\doi}[1]{doi: #1}\else
  \providecommand{\doi}{doi: \begingroup \urlstyle{rm}\Url}\fi

\bibitem[Beygelzimer et~al.(2017)Beygelzimer, Orabona, and Zhang]{SOBA}
A.~Beygelzimer, F.~Orabona, and C.~Zhang.
\newblock Efficient online bandit multiclass learning with
  $\tilde{O}(\sqrt{T})$ regret.
\newblock In \emph{ICML}, 2017.

\bibitem[Chang and Lin(2011)]{libsvm}
C.~C. Chang and C.~J. Lin.
\newblock Libsvm: A library for support vector machines.
\newblock In \emph{ACM}, 2011.

\bibitem[Crammer and Gentile(2011)]{Confidit}
K.~Crammer and C.~Gentile.
\newblock Multiclass classification with bandit feedback using adaptive
  regularization.
\newblock In \emph{ICML}, 2011.

\bibitem[Crammer et~al.(2006)Crammer, Dekel, Keshet, Shalev-Shwartz, and
  Singer]{PA}
K.~Crammer, O.~Dekel, J.~Keshet, S.~Shalev-Shwartz, and Y.~Singer.
\newblock Online passive-aggressive algorithms.
\newblock \emph{Journal of Machine Learning Research}, 2006.

\bibitem[Crammer et~al.(2013)Crammer, Kulesza, and Dredze]{Arow}
K.~Crammer, A.~Kulesza, and M.~Dredze.
\newblock Adaptive regularization of weight vectors.
\newblock \emph{Journal of Machine Learning Research}, 2013.

\bibitem[Dredze et~al.(2008)Dredze, Crammer, and Pereira]{CW}
M.~Dredze, K.~Crammer, and F.~Pereira.
\newblock Confidence-weighted linear classification.
\newblock In \emph{ICML}, 2008.

\bibitem[Duchi and Singer(2009)]{Fobos}
J.~Duchi and Y.~Singer.
\newblock Efficient online and batch learning using forward backward splitting.
\newblock \emph{Journal of Machine Learning Research}, 2009.

\bibitem[Duchi et~al.(2011)Duchi, , Hazan, and Singer]{Adagrad}
J.~Duchi, , E.~Hazan, and Y.~Singer.
\newblock Adaptive subgradient methods for online learning and stochastic
  optimization.
\newblock \emph{Journal of Machine Learning Research}, 2011.

\bibitem[Gentile and Orabona(2012)]{mr}
C.~Gentile and F.~Orabona.
\newblock On multilabel classification and ranking with partial feedback.
\newblock In \emph{NeurIPS}, 2012.

\bibitem[Hazan and Kale(2011)]{Newtron}
E.~Hazan and S.~Kale.
\newblock Newtron: an efficient bandit algorithm for online multiclass
  prediction.
\newblock In \emph{NeurIPS}, 2011.

\bibitem[Ishida et~al.(2017)Ishida, Niu, Hu, and Sugiyama]{complementary}
T.~Ishida, G.~Niu, W.~Hu, and M.~Sugiyama.
\newblock Learning from complementary labels.
\newblock In \emph{NeurIPS}, 2017.

\bibitem[John and Nello(2004)]{Kernelmethod}
S.~John and C.~Nello.
\newblock \emph{Kernel Methods for Pattern Analysis}.
\newblock Cambridge University Press, 2004.

\bibitem[Kakade et~al.(2008)Kakade, Shalev-Shwartz, and Tewari]{Banditron}
S.~Kakade, S.~Shalev-Shwartz, and A.~Tewari.
\newblock Efficient bandit algorithms for online multiclass prediction.
\newblock In \emph{ICML}, 2008.

\bibitem[Matsushima et~al.(2010)Matsushima, Shimizu, Yoshida, Ninomiya, and
  Nakagawa]{SPA}
S.~Matsushima, N.~Shimizu, K.~Yoshida, T.~Ninomiya, and H.~Nakagawa.
\newblock Exact passive-aggressive algorithm for multiclass classification
  using support class.
\newblock In \emph{SDM}, 2010.

\bibitem[Rosenblatt(1958)]{Perceptron}
F.~Rosenblatt.
\newblock The perceptron: A probabilistic model for information storage and
  organization in the brain.
\newblock \emph{Psychological Review}, 1958.

\bibitem[Shi et~al.(2016)Shi, Wang, Tian, Gan, and Wang]{CWB}
C.~Shi, X.~Wang, X.~Tian, X.~Gan, and X.~Wang.
\newblock Online multiclass learning with “bandit” feedback under a
  confidence-weighted approach.
\newblock In \emph{IEEE}, 2016.

\bibitem[Wang et~al.(2012)Wang, Zhao, and Hoi]{SCW}
J.~Wang, P.~Zhao, and S.C. Hoi.
\newblock Exact soft confidence-weighted learning.
\newblock In \emph{ICML}, 2012.

\bibitem[Wang et~al.(2010)Wang, Jin, and Valizadegan]{Exp_grad}
S.~Wang, R.~Jin, and H.~Valizadegan.
\newblock A potential-based framework for online multi-class learning with
  partial feedback.
\newblock In \emph{AISTATS}, 2010.

\bibitem[Xiao(2010)]{rda}
L.~Xiao.
\newblock Dual averaging methods for regularized stochastic learning and online
  optimization.
\newblock \emph{Journal of Machine Learning Research}, 2010.

\bibitem[Zhong and Daucé(2015)]{BPA}
H.~Zhong and E.~Daucé.
\newblock Passive-aggressive bounds in bandit feedback classification.
\newblock In \emph{ECML}, 2015.

\end{thebibliography}
    \clearpage
    \appendix
    \section{Proof of Theorem \ref{main_theorem}}
\begin{proof}
As in \citep{PA, SPA}, we define
\begin{equation}
\label{delta}
\Delta_t = \sum_{i \in \{1, \ldots, K\}} \left( \|{\bf w}_{i, t} - {\bf u}_{i}\|^2 - \|{\bf w}_{i, t+1} - {\bf u}_{i}\|^2 \right),
\end{equation}
and consider upper and lower bounds of
\begin{equation}
\sum_{t=1}^{T} \Delta_t.
\end{equation}

First, we derive an upper bound of (\ref{delta}). Using telescoping sum, we have the following inequality: 
\begin{equation}
\begin{split}
\label{upper}
\sum_{t=1}^{T} \Delta_t &= \sum_{t=1}^{T} \sum_{i \in \{1, \ldots, K\}} \left( \|{\bf w}_{i, t} - {\bf u}_{i}\|^2 - \|{\bf w}_{i, t+1} - {\bf u}_{i}\|^2 \right) \\
&= \sum_{i \in \{1, \ldots, K\}} \sum_{t=1}^{T} \left( \|{\bf w}_{i, t} - {\bf u}_{i}\|^2 - \|{\bf w}_{i, t+1} - {\bf u}_{i}\|^2 \right) \\
&= \sum_{i \in \{1, \ldots, K\}} \left( \|{\bf w}_{i, 1} - {\bf u}_{i}\|^2 - \|{\bf w}_{i, T+1} - {\bf u}_{T+1}\|^2 \right) \\
&\leq \sum_{i \in \{1, \ldots, K\}} \|{\bf w}_{i, 1} - {\bf u}_{i}\|^2 \\
&= \sum_{i \in \{1, \ldots, K\}} \|{\bf u}_{i}\|^2 \,\, ({\bf w}_{i, 1} = {\bf 0} \,\, \forall i \in \{1, \ldots, K\}).
\end{split}
\end{equation}

Next, we derive a lower bound of (\ref{delta}). When $M_t =$ True, CSPA uses SPA algorithm. In this case, it is shown in \citep{SPA} that 
\begin{equation}
\begin{split}
\Delta_t & \geq \left(\frac{|S_t| + 3}{4|S_t| + 4}\ell_t - \ell_t^{\star}\right)\frac{\ell_t}{\|{\bf x}_t\|^2}.
\end{split}
\end{equation}
Then, we gain the following bound:
\begin{equation}
\begin{split}
\label{SPA_bound}
\Delta_t & \geq \left(\frac{K + 3}{4(K + 1)}\ell_t - \ell_t^{\star}\right)\frac{\ell_t}{\|{\bf x}_t\|^2} \\
&= \left(\frac{K + 3}{4(K + 1)}\ell_t - \ell_t^{\star}\right)\frac{\ell_t}{R^2} \,\, (\because \|{\bf x}_t\|^2 = R^2).
\end{split}
\end{equation}

For the case $M_t =$ False, we apply the CPA update rule to (\ref{delta}), which yields, 
\begin{equation}
\begin{split}
\label{CPA_bound}
\Delta_t &= \sum_{i \in \{1, \ldots, K\}} \|{\bf w}_{i, t} - {\bf u}_{i}\|^2 - \sum_{i \in \{1, \ldots, K\} \backslash \{{\tilde y}_t\}} \|{\bf w}_{i, t} + \frac{\beta\ell_t}{K\|{\bf x}_t\|^2}{\bf x}_t - {\bf u}_{i}\|^2 \\ 
& \hspace{7cm} - \|{\bf w}_{{\tilde y}_t, t} - \frac{\beta(K-1)\ell_t}{K\|{\bf x}_t\|^2}{\bf x}_t - {\bf u}_{{\tilde y}_t}\|^2 \\
&= -2 \sum_{i \in \{1, \ldots, K\} \backslash \{{\tilde y}_t\}} \left\{\left({\bf w}_{i, t} - {\bf u}_{i}\right)^{\top}\left(\frac{\beta\ell_t}{K\|{\bf x}_t\|^2}{\bf x}_t\right)\right\} \\ 
& \hspace{0.5cm} + 2(K-1)\left\{({\bf w}_{{\tilde y}_t, t} - {\bf u}_{{\tilde y}_t})^{\top}\left(\frac{\beta\ell_t}{K\|{\bf x}_t\|^2}{\bf x}_t\right)\right\} - \frac{\beta^2(K-1)\ell_t^2}{K^2\|{\bf x}_t\|^2} - \frac{\beta^2(K-1)^2\ell_t^2}{K^2\|{\bf x}_t\|^2} \\
&= \frac{2\beta\ell_t}{K\|{\bf x}_t\|^2} \left[ \sum_{i \in \{1, \ldots, K\} \backslash \{{\tilde y}_t\}} \left\{ \left(1 + {\bf w}_{{\tilde y}_t}^{\top}{\bf x}_t - {\bf w}_{i}^{\top}{\bf x}_t \right) - \left(1 + {\bf u}_{{\tilde y}_t}^{\top}{\bf x}_t - {\bf u}_{i}^{\top}{\bf x}_t \right) \right\} \right] \\
& \hspace{10cm} - \frac{K-1}{K}\frac{\beta^2\ell_t^2}{\|{\bf x}_t\|^2} \\
&\geq \frac{2\beta\ell_t}{K\|{\bf x}_t\|^2} \left[ \sum_{i \in \{1, \ldots, K\} \backslash \{{\tilde y}_t\}} \left\{ \ell_t - \left(1 + {\bf u}_{{\tilde y}_t}^{\top}{\bf x}_t - {\bf u}_{i}^{\top}{\bf x}_t \right) \right\} \right] - \frac{K-1}{K}\frac{\beta^2\ell_t^2}{\|{\bf x}_t\|^2} \\
& \hspace{10cm} (\because \mathrm{\, definition \, of \,} \ell_t) \\
&= \frac{2\beta\ell_t}{K\|{\bf x}_t\|^2} \left[(K-1)\ell_t - \sum_{i \in \{1, \ldots, K\} \backslash \{{\tilde y}_t\}} \left(1 + {\bf u}_{{\tilde y}_t}^{\top}{\bf x}_t - {\bf u}_{i}^{\top}{\bf x}_t \right) \right] - \frac{K-1}{K}\frac{\beta^2\ell_t^2}{\|{\bf x}_t\|^2} \\
&= \frac{2\beta\ell_t}{K\|{\bf x}_t\|^2} \left[(K-1)\ell_t - \sum_{i \in \{1, \ldots, K\} \backslash \{{\tilde y}_t, y_t\}} \left(1 + {\bf u}_{{\tilde y}_t}^{\top}{\bf x}_t - {\bf u}_{i}^{\top}{\bf x}_t \right) - \left(1 + {\bf u}_{{\tilde y}_t}^{\top}{\bf x}_t - {\bf u}_{y_t}^{\top}{\bf x}_t \right) \right] \\ 
& \hspace{11cm} - \frac{K-1}{K}\frac{\beta^2\ell_t^2}{\|{\bf x}_t\|^2} \\
&\geq \frac{2\beta\ell_t}{K\|{\bf x}_t\|^2} \left[(K-1)\ell_t - \sum_{i \in \{1, \ldots, K\} \backslash \{{\tilde y}_t, y_t\}} \left(1 + {\bf u}_{{\tilde y}_t}^{\top}{\bf x}_t - {\bf u}_{i}^{\top}{\bf x}_t \right) - \ell_t^{\star} \right] - \frac{K-1}{K}\frac{\beta^2\ell_t^2}{\|{\bf x}_t\|^2} \\
& \hspace{10cm} (\because \mathrm{\, definition \, of \,} \ell_t^{\star}) \\
&\geq \frac{2\beta\ell_t}{K\|{\bf x}_t\|^2} \left\{(K-1)\ell_t - (K-2) - \alpha - \ell_t^{\star} \right\} - \frac{K-1}{K}\frac{\beta^2\ell_t^2}{\|{\bf x}_t\|^2} \,\, (\because \textrm{ assumption of } (\ref{assumption})) \\
&= \frac{\beta(K-1)\ell_t^2}{K\|{\bf x}_t\|^2}(2-\beta) - \frac{2\beta(K-2+\alpha)\ell_t}{K\|{\bf x}_t\|^2} - \frac{2\beta\ell_t\ell_t^{\star}}{K\|{\bf x}_t\|^2} \\
&\geq \frac{\beta(K-1)\ell_t^2}{K\|{\bf x}_t\|^2}(2-\beta) - \frac{2\beta(K-2+\alpha)\ell_t^2}{K\|{\bf x}_t\|^2} - \frac{2\beta\ell_t\ell_t^{\star}}{K\|{\bf x}_t\|^2} \,\, (\because \ell_t \geq 1 \textrm{ by } (\ref{CPA_l})) \\
&= \frac{\beta\ell_t^2}{K\|{\bf x}_t\|^2}\left\{(K-1)(2-\beta) - 2(K-2 + \alpha)\right\} - \frac{2\beta\ell_t\ell_t^{\star}}{K\|{\bf x}_t\|^2} \\
&= \frac{\beta\left\{2(1-\alpha) - (K-1)\beta\right\}}{KR^2}\ell_t^2 - \frac{2\beta\ell_t\ell_t^{\star}}{KR^2} \,\, (\because \|{\bf x}_t\|^2 = R^2).
\end{split}
\end{equation}

Note that $2(1-\alpha) - (K-1)\beta > 0$ by the assumption of (\ref{beta_assumption}). Then, we introduce $\gamma$ defined as follows:
\begin{equation}
\gamma = 2(1-\alpha) - (K-1)\beta > 0.
\end{equation}

Combining (\ref{SPA_bound}) and (\ref{CPA_bound}), we obtain the following bound:
\begin{equation}
\begin{split}
\Delta_t & \geq \min \left\{\frac{\beta\gamma}{KR^2}\ell_t^2 - \frac{2\beta\ell_t\ell_t^{\star}}{KR^2}, \frac{KR^2 + 3}{4(K + 1)R^2}\ell_t^2 - \frac{1}{R^2}\ell_t^{\star}\ell_t\right\} \\
& \geq \min\left\{\frac{\beta\gamma}{KR^2}, \frac{K+3}{4(K+1)R^2}\right\}\ell_t^2 - \max\left\{\frac{2\beta}{KR^2}, \frac{1}{R^2}\right\}\ell_t\ell_t^{\star},
\end{split}
\end{equation}
which is equivalent to 
\begin{equation}
\ell_t^2 \leq \frac{1}{\min\left\{\frac{\beta\gamma}{KR^2}, \frac{K+3}{4(K+1)R^2}\right\}}\left\{\Delta_t + \max\left\{\frac{2\beta}{KR^2}, \frac{1}{R^2}\right\}\ell_t\ell_t^{\star}\right\}.
\end{equation}
Taking the sum over $t = 1, \ldots, K$ and combining it with (\ref{upper}), we obtain
\begin{equation}
\sum_{t=1}^T \ell_t^2 \leq \frac{1}{\min\left\{\frac{\beta\gamma}{KR^2}, \frac{K+3}{4(K+1)R^2}\right\}}\left\{\sum_{i \in \{1, \ldots, K\}} \|{\bf u}_i\|^2 + \max\left\{\frac{2\beta}{KR^2}, \frac{1}{R^2}\right\}\sum_{t=1}^{T}\ell_t\ell_t^{\star} \right\}.
\end{equation}

Here, we define $L$ and $L^{\star}$ as follows:
\begin{equation}
\begin{split}
& L = \sqrt{\sum_{t=1}^{T} \ell_t^2}, \\
& L^{\star} = \sqrt{\sum_{t=1}^{T} (\ell_t^{\star})^2}.
\end{split}
\end{equation}
Then, using Cauchy-Schwartz inequality, $\sum_{t=1}^{T}\ell_t\ell_t^{\star} \leq LL^{\star}$ holds, so the following inequality is obtained:
\begin{equation}
L^2 \leq \frac{1}{\min\left\{\frac{\beta\gamma}{KR^2}, \frac{K+3}{4(K+1)R^2}\right\}}\left\{\sum_{i \in \{1, \ldots, K\}} \|{\bf u}_i\|^2 + \max\left\{\frac{2\beta}{KR^2}, \frac{1}{R^2}\right\}LL^{\star} \right\}, 
\end{equation}
which is equivalent to
\begin{equation}
\label{poli}
\chi L^2 - \psi L^{\star}L - \sum_{i \in \{1, \ldots, K\}} \|{\bf u}_i\|^2 \leq 0
\end{equation}
where 
\begin{equation}
\begin{split}
& \chi = \frac{1}{R^2}\min\left\{\frac{\beta\gamma}{K}, \frac{K+3}{4(K+1)}\right\}, \\
& \psi = \frac{1}{R^2}\max\left\{\frac{2\beta}{K}, 1\right\}.
\end{split}
\end{equation}
We regard (\ref{poli}) as a quadratic equation with respect to $L$, we obtain
\begin{equation}
\begin{split}
L & \leq \frac{\psi L^{\star} + \sqrt{\psi^2 (L^{\star})^2+ 4\chi\sum_{i \in \{1, \ldots, K\}}\|{\bf u}_i\|^2}}{2\chi} \\
& \leq \frac{\psi}{\chi}L^{\star} + \sqrt{\frac{\sum_{i \in \{1, \ldots, K\}}\|{\bf u}_i\|^2}{\chi}} \,\, (\because \sqrt{x+y} \leq \sqrt{x} + \sqrt{y}).
\end{split}
\end{equation}
Then, the following holds:
\begin{equation}
\begin{split}
\frac{2\beta}{K} &< \frac{4(1-\alpha)}{K(K-1)} \\
      &< \frac{4}{K(K-1)} \,\, (\because \alpha > 0)\\
      &\leq \frac{2}{3} \,\,(\because K \geq 3), \\
\end{split}
\end{equation}
which means that $\psi$ is always equal to $\frac{1}{R^2}$.

When $\beta = \frac{1-\alpha}{K-1}$, 
\begin{equation}
\begin{split}
\frac{\beta\gamma}{K} & = \frac{\beta(2(1-\alpha) - (K-1)\beta)}{K} \\
& = \frac{(1-\alpha)^2}{K(K-1)}, \\
& < \frac{K+3}{4(K+1)},
\end{split}
\end{equation}
from $K \geq 3$. Therefore, we have:
\begin{equation}
\chi = \frac{1}{R^2}\frac{(1-\alpha)^2}{K(K-1)}.
\end{equation}

Then, we obtain
\begin{equation}
\sum_{t=1}^{T} \ell_t^2 \leq \left(\frac{K(K-1)}{(1-\alpha)^2}\sqrt{\sum_{i=1}^{T} (\ell_t^{\star})^2} + \frac{R\sqrt{K(K-1)}}{1-\alpha}\sqrt{\sum_{i \in \{1, \ldots, K\}}\|{\bf u}_i\|^2} \right)^2.
\end{equation}

\end{proof}

    \section{More Experimental Results}
\begin{table*}[htb]
\centering
\caption{Average and standard deviation of the ratio of correct \emph{proposed} labels in linear function case in percentage over ten trials. The methods with best $5\%$ t-test results are in boldface.}
\label{table1}
\vspace{0.1in}
\scalebox{1.0}{
\begin{tabular}{|c||l|l|l||c|c|c|c|}
\hline
& Labels & Instances & Features
&\multicolumn{1}{c|}{\hspace{0.1cm}CSPA\,}&\multicolumn{1}{c|}{Banditron}&\multicolumn{1}{c|}{Confidit}&\multicolumn{1}{c|}{\hspace{0.25cm}BPA\,\,\,} \\
\hline
\multirow{2}{40pt}{20News} & \multirow{2}{8pt}{20} & \multirow{2}{22pt}{15,935} & \multirow{2}{8pt}{62,061} & \multicolumn{1}{c|}{\textbf{66.7}} & \multicolumn{1}{c|}{23.6} & \multicolumn{1}{c|}{55.9} & \multicolumn{1}{c|}{63.0} \\
& & & & \multicolumn{1}{c|}{(0.5)} & \multicolumn{1}{c|}{(1.5)} & \multicolumn{1}{c|}{(2.4)} & \multicolumn{1}{c|}{(1.4)} \\
\hline
\multirow{2}{40pt}{Sector} & \multirow{2}{8pt}{105} & \multirow{2}{22pt}{6412} & \multirow{2}{8pt}{55,197} & \multicolumn{1}{c|}{\textbf{8.83}} & \multicolumn{1}{c|}{2.67} & \multicolumn{1}{c|}{7.43} & \multicolumn{1}{c|}{7.43} \\
& & & & (0.99) & (0.33) & (1.09) & \multicolumn{1}{c|}{(1.14)} \\
\hline
\multirow{2}{40pt}{Vehicle} & \multirow{2}{8pt}{4} & \multirow{2}{22pt}{846}  & \multirow{2}{8pt}{18} & \multicolumn{1}{c|}{\textbf{49.3}} & \multicolumn{1}{c|}{35.9} & \multicolumn{1}{c|}{47.4} & \multicolumn{1}{c|}{\textbf{48.1}} \\
& & & & (1.7) &(2.6) & (1.8) & \multicolumn{1}{c|}{(1.8)} \\
\hline
\multirow{2}{40pt}{Shuttle} & \multirow{2}{8pt}{7} & \multirow{2}{22pt}{43,500}  & \multirow{2}{8pt}{9} & \multicolumn{1}{c|}{\textbf{95.3}} & \multicolumn{1}{c|}{86.6} & \multicolumn{1}{c|}{86.6} & \multicolumn{1}{c|}{87.1} \\
& & & & (0.1) & (3.8) & (0.0) & \multicolumn{1}{c|}{(1.7)} \\
\hline
\multirow{2}{40pt}{USPS} & \multirow{2}{8pt}{10} & \multirow{2}{22pt}{7,291} & \multirow{2}{8pt}{256} & \multicolumn{1}{c|}{\textbf{84.9}} & \multicolumn{1}{c|}{48.5} & \multicolumn{1}{c|}{78.5} & \multicolumn{1}{c|}{81.4} \\ 
& & & & (0.3) & (5.5) & (1.9) & \multicolumn{1}{c|}{(1.5)} \\
\hline
\multirow{2}{40pt}{Pendigits} & \multirow{2}{8pt}{10} & \multirow{2}{22pt}{7,494} & \multirow{2}{8pt}{16} & \multicolumn{1}{c|}{\textbf{79.7}} & \multicolumn{1}{c|}{32.7} & \multicolumn{1}{c|}{60.7} & \multicolumn{1}{c|}{70.7} \\
& & & &  (0.4) & (3.0) & (2.4) & \multicolumn{1}{c|}{(1.5)} \\
\hline
\end{tabular}
}
\end{table*}

\begin{table*}[htb]
\centering
\caption{Average and standard deviation of the ratio of correct \emph{proposed} labels in nonlinear function case in percentage over ten trials. Gaussian kernel with support set of size 700 are used as a kernel. The methods with best $5\%$ t-test results are in boldface.}
\label{table2}
\vspace{0.1in}
\scalebox{1.0}{
\begin{tabular}{|c||l|l|l||c|c|c|c|}
\hline
& Labels & Instances & Features
&\multicolumn{1}{c|}{\hspace{0.1cm}CSPA\hspace{0.2cm}}&\multicolumn{1}{c|}{Banditron}&\multicolumn{1}{c|}{Confidit}&\multicolumn{1}{c|}{\hspace{0.25cm}BPA\,\,\,} \\
\hline
\multirow{2}{50pt}{Satimage} & \multirow{2}{8pt}{6} & \multirow{2}{22pt}{4435} & \multirow{2}{8pt}{36} & \multicolumn{1}{c|}{\textbf{86.2}} & \multicolumn{1}{c|}{69.3} & \multicolumn{1}{c|}{80.2} & \multicolumn{1}{c|}{81.9}\\
& & & & \multicolumn{1}{c|}{(0.3)} & \multicolumn{1}{c|}{(1.1)} & \multicolumn{1}{c|}{(0.3)} & \multicolumn{1}{c|}{(0.8)} \\
\hline
\multirow{2}{50pt}{MNIST} & \multirow{2}{8pt}{10} & \multirow{2}{22pt}{60,000} & \multirow{2}{8pt}{784} & \multicolumn{1}{c|}{\textbf{91.8}} & \multicolumn{1}{c|}{69.0} & \multicolumn{1}{c|}{87.6} & \multicolumn{1}{c|}{89.7} \\
& & & & (0.2) & (1.0) & (0.2) & \multicolumn{1}{c|}{(0.2)} \\
\hline
\multirow{2}{50pt}{Letter} & \multirow{2}{8pt}{26} & \multirow{2}{22pt}{15,000}  & \multirow{2}{8pt}{16} & \multicolumn{1}{c|}{\textbf{62.4}} & \multicolumn{1}{c|}{26.8} & \multicolumn{1}{c|}{36.0} & \multicolumn{1}{c|}{50.7} \\
& & & & (1.6) & (6.3) & (0.9) & \multicolumn{1}{c|}{(0.5)} \\
\hline
\multirow{2}{50pt}{Segment} & \multirow{2}{8pt}{7} & \multirow{2}{22pt}{2310}  & \multirow{2}{8pt}{19} & \multicolumn{1}{c|}{\textbf{90.1}} & \multicolumn{1}{c|}{76.2} & \multicolumn{1}{c|}{74.9} & \multicolumn{1}{c|}{86.7} \\
& & & &(0.6)&(0.8) & (1.7)& \multicolumn{1}{c|}{(0.6)} \\
\hline
\multirow{2}{50pt}{Vowel} & \multirow{2}{8pt}{11} & \multirow{2}{22pt}{528} & \multirow{2}{8pt}{10} & \multicolumn{1}{c|}{\textbf{41.8}} & \multicolumn{1}{c|}{\textbf{43.1}} & \multicolumn{1}{c|}{27.1} & \multicolumn{1}{c|}{38.6} \\ 
& & & & (4.6) &(4.8) &(1.3) & \multicolumn{1}{c|}{(2.7)} \\
\hline
\multirow{2}{50pt}{Sensorless} & \multirow{2}{8pt}{11} & \multirow{2}{22pt}{58509} & \multirow{2}{8pt}{48} & \multicolumn{1}{c|}{\textbf{92.6}} & \multicolumn{1}{c|}{68.9} & \multicolumn{1}{c|}{67.7} & \multicolumn{1}{c|}{87.6} \\
& & & &(0.3) & (1.2) & (0.5) & \multicolumn{1}{c|}{(1.2)} \\
\hline
\end{tabular}
}
\end{table*}

\end{document}